\documentclass[twocolumn]{article}

\usepackage{microtype}
\usepackage{graphicx}
\usepackage{subfigure}
\usepackage{booktabs}
\usepackage{xcolor}
\usepackage{amsmath}
\usepackage{amsthm}
\usepackage{graphicx}
\usepackage{color}
\usepackage{amssymb}
\usepackage{hyperref}
\hypersetup{
     colorlinks = true,
     linkcolor = green!60!black,
     anchorcolor = green!50!black,
     citecolor = green!60!black,
     filecolor = green!60!black,
     urlcolor = green!60!black
     }
\usepackage[nameinlink,noabbrev]{cleveref}

\usepackage[backend = biber,
    firstinits=true,
    hyperref = true,
    maxbibnames = 10,
    sortcites,
    maxcitenames=10,
    doi=false,
    url=false,
    sortcites,
    ]{biblatex}

\newcommand{\mycolor}{black!40!blue!66!green!40!white}

\makeatletter
\let\blx@rerun@biber\relax
\makeatother

\newcommand{\norm}[1]{\left\lVert#1\right\rVert}

\newtheorem{theorem}{Theorem}
\newtheorem{lemma}[theorem]{Lemma}

\usepackage[maxfloats=40]{morefloats}

\usepackage{tikz}
\tikzset{every picture/.style={line width=0.75pt}}
\usetikzlibrary{matrix,fit,calc,trees,positioning,arrows,chains,shapes.geometric,shapes,angles,quotes}

\DeclareMathOperator*{\argmax}{arg\,max}
\DeclareMathOperator*{\argmin}{arg\,min}

\newcommand{\F}{\mathcal{F}}
\newcommand{\Ftwo}{\mathcal{F}^{(2)}}
\newcommand{\Fone}{\mathcal{F}^{(1)}}

\newcommand{\NF}{\mathcal{NF}}
\newcommand{\NFtwo}{\mathcal{NF}^{(2)}}
\newcommand{\NFone}{\mathcal{NF}^{(1)}}

\newcommand{\pol}{\mathcal{Q}}

\newcommand{\dist}{\mathrm{dist}}

\newcommand{\mX}{\mathcal{X}}
\newcommand{\mXb}{\bar{\mathcal{X}}}

\title{\textbf{Detection of Iterative Adversarial Attacks\\ via Counter Attack}\medskip}

\author{Matthias Rottmann\footnote{University of Wuppertal, School of Mathematics and Natural Sciences, \{rottmann,kmaag,hgottsch\}@uni-wuppertal.de},
Kira Maag\footnotemark[1],
Mathis Peyron\footnote{ENSEEIHT Toulouse, Department of HPC and Big Data, mathis.peyron@etu.enseeiht.fr},
Nata\v{s}a Kreji\'{c}\footnote{University of Novi Sad, Faculty of Sciences, natasa@dmi.uns.ac.rs} \, and Hanno Gottschalk\footnotemark[1]}

\date{}

\begin{document}

\maketitle

\begin{abstract}
Deep neural networks (DNNs) have proven to be powerful tools for processing unstructured data. However for high-dimensional data, like images, they are inherently vulnerable to adversarial attacks. Small almost invisible perturbations added to the input can be used to fool DNNs. Various attacks, hardening methods and detection methods have been introduced in recent years. Notoriously, Carlini-Wagner (CW) type attacks computed by iterative minimization belong to those that are most difficult to detect. 
In this work we outline a mathematical proof that the CW attack can be used as a detector itself. That is,  
under certain assumptions and in the limit of attack iterations this detector provides asymptotically optimal separation of original and attacked images. In numerical experiments, we experimentally validate this statement and furthermore obtain AUROC values up to $99.73\%$ on CIFAR10 and ImageNet. This is in the upper part of the spectrum of current state-of-the-art detection rates for CW attacks.
\end{abstract}

\section{Introduction}

For many applications, deep learning has shown to outperform concurring machine learning approaches by far \cite{NIPS2012_4824,Simonyan14c,DBLP:journals/corr/HeZRS15}. Especially, when working with high-dimensional input data like images, deep neural networks (DNNs) show impressive results. As discovered by 
\cite{Szegedy2013IntriguingPO}, this remarkable performance comes with a downside. Very small (noise-like) perturbation added to an input image can result in incorrect predictions with high confidence \cite{Szegedy2013IntriguingPO,Goodfellow2014ExplainingAH}.
Such adversarial attacks are usually crafted by performing a projected gradient descent method for solving a constrained optimization problem. The optimization problem is formulated as the least change of the input image that yields a change in the class predicted by the DNN. The new class can be either a class of choice (targeted attack) or an arbitrary but different one (untargeted attack).
Many other types of attacks such as the fast signed gradient method \cite{Goodfellow2014ExplainingAH} and DeepFool \cite{DBLP:journals/corr/Moosavi-Dezfooli15} have been introduced, but these methods do not fool DNNs reliably.
Carlini \& Wagner (CW) extended \cite{Szegedy2013IntriguingPO} with a method that reliably attacks deep neural networks
while controlling important features of the attack like sparsity and maximum size over all pixels, see \cite{DBLP:journals/corr/CarliniW16a}. This method aims at finding a targeted or an untargeted attack where the distance between the attacked image and the original one is minimal with respect to a chosen $\ell_p$ distance. Distances of choice within in CW-type frameworks are mostly $\ell_p$ with $p=0,2,\infty$. Mixtures of $p=1,2$ have also been proposed by \cite{Chen2017EADEA}. Except for the $\ell_0$ distance, these attacks are perceptually extremely hard to detect, cf.~\cref{fig:attcifar}. Minimization of the $\ell_0$ distance minimizes the number of pixels changed, but also changes these pixels maximally and the resulting spikes are easy to detect. This changes for $p>0$.
Typically, one distinguishes between three different attack scenarios.
\emph{White box} attack: the attacker has access to the DNN's parameters / the whole framework including defense strategies.
\emph{Black box} attack: the attacker does not know the DNN's parameters / the whole framework including defense strategies.
\emph{Gray box}  attack: in-between white and black box, e.g.\ the attacker might know the framework but not the parameters used.
The CW attack currently is one of the most efficient white-box attacks.

\subsection{Defense Methods}
Several defense mechanisms have been proposed to either harden neural networks or to detect adversarial attacks. One such hardening method is the so-called \emph{defensive distillation}~\cite{DBLP:journals/corr/PapernotMWJS15}. With a single re-training step this method provides strong security, however not against CW attacks.
Training for robustness via adversarial training is another popular defense approach, see e.g.~\cite{Goodfellow2014ExplainingAH,Madry2017TowardsDL,pmlr-v97-wang19i,DBLP:journals/corr/KurakinGB16a,Metzen2017OnDA}. See also \cite{REN2020346} for an overview regarding defense methods.
Most of these methods are not able to deal with attacks based on iterative methods for constrained minimization like CW attacks. In a white box setting the hardened networks can still be susceptible to fooling.

\subsection{Detection Methods}
There are numerous detection methods for adversarial attacks. In many works, it has been observed and exploited that adversarial attacks are less robust to random noise than non-attacked clean samples, see e.g.~\cite{DBLP:journals/corr/AthalyeS17,DBLP:journals/corr/abs-1711-01991,DBLP:journals/corr/abs-1904-00689,DBLP:journals/corr/abs-1902-04818}. This robustness issue can be utilized by detection and defense methods. In \cite{DBLP:journals/corr/abs-1902-04818}, a statistical test is proposed for a white-box setup. Statistics are obtained under random corruptions of the inputs and observing the corresponding softmax probability distributions.
An adaptive noise reduction approach for detection of adversarial examples is presented in \cite{DBLP:journals/corr/LiangLSLSW17}. JPEG compression \cite{DBLP:journals/corr/abs-1803-00940,DBLP:journals/corr/abs-1803-05787}, similar input transformations \cite{DBLP:journals/corr/abs-1711-00117} and other filtering techniques \cite{Lee2018DefensiveDM} have demonstrated to filter out many types of adversarial attacks as well. 
%
%
Some approaches also work with several images or sequences of images to detect adversarial images, see e.g.~\cite{DBLP:journals/corr/abs-1907-05587,DBLP:journals/corr/GrosseMP0M17}.
Approaches that not only take input or output layers into account, but also hidden layer statistics, are introduced in \cite{Carrara_2018_ECCV_Workshops,Zheng:2018:RDA:3327757.3327888}.
Auxiliary models trained for robustness and equipped for detecting adversarial examples are presented in \cite{DBLP:journals/corr/abs-1905-11475}.
Recently, GANs have demonstrated to defend DNNs against adversarial attacks very well, see \cite{DBLP:journals/corr/abs-1805-06605}. This approach called defense-GAN iteratively filters out the adversarial perturbation. The filtered result is then presented to the original classifier. 
Semantic concepts in image classification tasks pre-select image data that only cover a tiny fraction of the expressive power of rgb images. Therefore training datasets can be embedded in lower dimensional manifolds and DNNs are only trained with data within the manifolds and behave arbitrarily in perpendicular directions. These are also the directions used by adversarial perturbations. Thus, the manifold distance of adversarial examples can be used as criterion for detection in~\cite{8599691}.
As opposed to many of the adversarial training and hardening approaches, most of the works mentioned in this paragraph are able to detect CW attacks with AUROC values up to 99\% \cite{DBLP:journals/corr/LuIF17}. For an overview we refer to \cite{DBLP:journals/ijautcomp/XuMLDLTJ20}.


The idea of detecting adversarial attacks by applying another attack (\emph{counter attack}) to both attacked and non-attacked data (images as well as text), instead of applying random perturbations as previously discussed, was first proposed in \cite{Worzyk1,Worzyk2} (termed $\mathit{adv}^{-1}$) and re-invented by \cite{undercover} (termed \emph{undercover attack}). The idea is to measure the perturbation strength of the counter attack in a given $\ell_p$ norm and then to discriminate between $\ell_p$ norms of perturbations for images that were attacked previously and images that are still original (non-attacked). The motivation of this approach is similar to the one for utilizing random noise, i.e., attacked images tend to be less robust to adversarial attacks than original images.
Indeed there is a difference between the methods proposed by \cite{Worzyk1,Worzyk2} and \cite{undercover}. While the former measure counter attack perturbation norms on the input space, the latter do so on the space of softmax probabilities.
All three works are of empirical nature and present results for a range of attack methods (including FGSM and CW) where they apply the same attack another time (to both attacked and non-attacked data) and also perform cross method attacks, achieving detection accuracies that in many test are state-of-the-art and beyond.
As opposed to many other types of detection methods, counter attack methods tend to detect stronger attacks such as the CW attack much more reliably than less strong attacks. Here, strong refers to attacks that fool DNNs reliably while producing very small perturbations.

\subsection{Our Contribution} \label{sec:ourcontr}
In this paper we establish theory for the CW attack and the counter attack framework \cite{Worzyk1,Worzyk2} applied to the CW attack. As opposed to other types of attack (not based on constraint optimization) the inherent properties of the CW  attack constitute an alternative motivation for the counter attack method that turns out to lead to provable detection:
\begin{figure}[tb]
    \centering
    \scalebox{0.55}{
    \begin{tikzpicture}[thick,font=\sffamily\Large]
\draw[fill=black!0!white,color=black!0!white] (0.5,-3) --  (-6,-6) -- (-6,1.5) -- (6,1.5) -- (6,-4) -- cycle;
\draw[fill=black!20!orange!15!white,color=black!20!orange!15!white] (0.5,-3) --  (-6,-6) -- (-6,-6.5) -- (6,-6.5) -- (6,-4) -- cycle;

\draw[very thick] (0.5,-3) --  (-6,-6);
\draw[very thick] (0.5,-3) --  (6,-4);

\node [fill, draw, circle, minimum width=3pt, inner sep=0pt, pin={[outer sep=1.5pt, pin distance=120pt]175:$x^*$}] at (0.5,-3) {};
\node [fill, draw, circle, minimum width=3pt, inner sep=0pt, pin={[outer sep=1.5pt, pin distance=120pt]180:$x_0$}] (x0) at (0.465,0.5) {};
\node [fill, draw, circle, minimum width=3pt, inner sep=0pt, pin={[outer sep=1.5pt, pin distance=120pt]185:$x_k$}] (xk) at (0.55,-3.5) {};
\node [fill, draw, circle, minimum width=3pt, inner sep=0pt, pin={[outer sep=1.5pt, pin distance=80pt]10:$x_{k,j}$}] (xkj) at (1,-2.5) {};

\draw[->,black,dashed] (x0) to [bend right] node [midway,left]{attack} (xk) ;
\draw[->,black,dashed] (xk) to [bend right] node [midway,above right]{counter attack} (xkj) ;

\node at (4.5,-5.5) {class $j$};
\node at (4.5,0.5) {class $i$};
\end{tikzpicture}
}
    \caption{Sketch of the action of the counter attack, $x^*$ is the stationary point of the first attack, $x_k$ the final iterate of the first attack and $x_{k,j}$ the final iterate of the second attack.}
    \label{fig:repeatedAttack}
\end{figure}
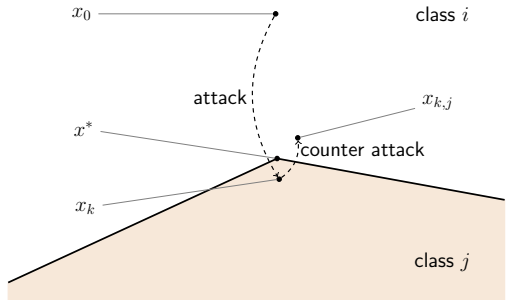
The CW attack minimizes the perturbation in a chosen norm such that the class prediction changes. Thus, an iterative sequence is supposed to terminate closely behind the decision boundary. This also holds for the counter attack. Therefore, in a statistical sense, the counter attack is supposed to produce much smaller perturbations on already attacked images than on original images, see also \cref{fig:repeatedAttack}. This point of view does not necessarily generalize to other types of attacks that are not based on constrained optimization.
With this work, we are the first to mathematically prove the following results for the $\ell_2$ CW attack: 
\begin{itemize}
\item For specific choice of penalty term the CW attack converges to a stationary point.
\item 
If the CW attack is successful (the predicted class of the attacked image is different from the prediction for the original image), the CW attack converges to a stationary point on the class boundary.
\item 
If the counter attack is started at an iteration of the first attack that is sufficiently close to the limit point of the attack sequence,
the counter attack asymptotically perfectly separates attacked and non-attacked images by means perturbation strength, i.e., by $\ell_p$ norm of the perturbation produced by the counter attack.
\end{itemize}
We are therefore able to explain many of the results found in  \cite{Worzyk1,Worzyk2,undercover}. It is exactly the optimality of the CW attack that leaves a treacherous footprint in the attacked data. Note that we do not claim that the counter attack cannot be bypassed as well, however our work contributes to understanding the nature of CW attacks and the corresponding counter attack.

We complement the mathematical theory with numerical experiments on 2D examples where 
the assumptions of our proofs are satisfied and study the relevant parameters such as the number of iterations of the primary attack, initial learning rate of the secondary attack and others. We
consecutively step by step 
relax the assumptions, showing that the detection accuracy in terms of area under the receiver operator curve (AUROC) remains still close to $100\%$.  This is complemented with a study on the dependence of the problem dimension. Lastly, we present numerical results for the CW (counter) attack with default parameters applied to CIFAR10 and ImageNet that extend the studies found in \cite{Worzyk1,Worzyk2,undercover}. We again study the dependence on the number of primary attack iterations and furthermore perform cross attacks for different $\ell_p$ norms that maintain high AUROC values of up to $99.73\%$.

\subsection{Related Work}
Works on counter attacks published so far \cite{Worzyk1,Worzyk2,undercover} treat the counter attack idea solely experimentally. In general, the counter attack can be viewed as a specialization of the methods that proceed analogously but use randomized noise instead \cite{DBLP:journals/corr/AthalyeS17,DBLP:journals/corr/abs-1711-01991,DBLP:journals/corr/abs-1904-00689,DBLP:journals/corr/abs-1902-04818}. While \cite{Worzyk1,Worzyk2,undercover} present results for different attacks utilized in the counter-attack framework, we focus on the CW attack with theoretical foundation. While we methodologically proceed similarly to \cite{Worzyk1}, that work studies the method more broadly and does not consider parameter dependencies, such as the number of iterations and penalization strengths. We show in the present work that these parameters are crucial for guaranteeing the success of the counter attack. The works \cite{Worzyk2,undercover} put additional classifiers on top and consider different input metrics, e.g.\ different $\ell_p$ norms of the perturbation, aiming at improving the empirical results. In our experiments we focus on providing additional insight specifically tailored to our theoretical findings. We complement this with experiments on CIFAR10 and ImageNet. The latter has not yet been considered in counter attack experiments and w.r.t.\ the former we focus on the parameter dependency and cross attacks between different $\ell_p$ norms in the optimization objective.
Neither has been considered in the related works.
Conceptually, \cite{8599691} can also be regarded as close to our approach. The authors measure the distance from a manifold containing the data, therefore a model of the manifold is learned. Our method can be interpreted as thresholding on the distance that is required to find the closest decision boundary outside the given manifold that contains the image data. Unlike in \cite{8599691}, our approach does not require a model of the manifold and is rather based the intrinsic properties of CW attacks and other constraint minimization attacks. 


\begin{figure}[t]
\centering
\scalebox{0.81}{
\begin{tabular}{r c | c }
 $\ell_0$ &
 \includegraphics[width=1.3cm]{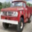} 
 \includegraphics[width=1.3cm]{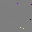}  \includegraphics[width=1.3cm]{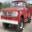}
 &
 \includegraphics[width=1.3cm]{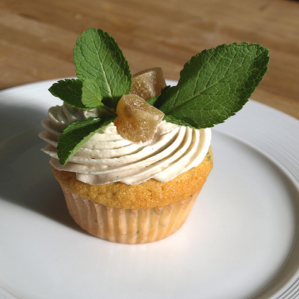}
 \includegraphics[width=1.3cm]{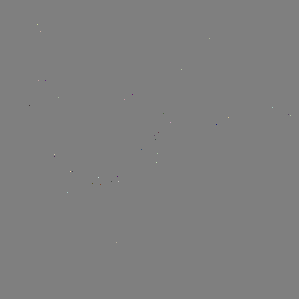}
 \includegraphics[width=1.3cm]{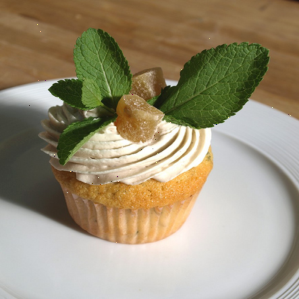}
 \\ 
 $\ell_2$&
 \includegraphics[width=1.3cm]{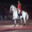}
 \includegraphics[width=1.3cm]{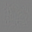} \includegraphics[width=1.3cm]{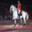}
 &
 \includegraphics[width=1.3cm]{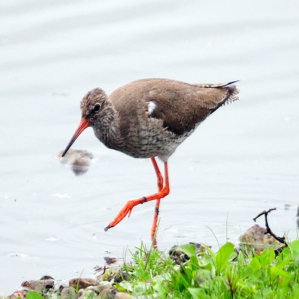}
 \includegraphics[width=1.3cm]{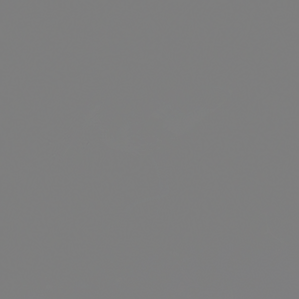} \includegraphics[width=1.3cm]{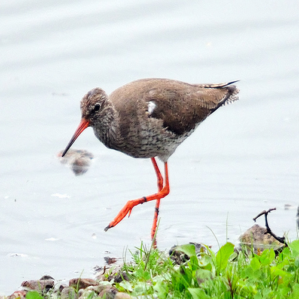}
 \\
 $\ell_\infty$&
 \includegraphics[width=1.3cm]{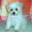}
 \includegraphics[width=1.3cm]{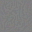}
 \includegraphics[width=1.3cm]{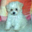}
 &
 \includegraphics[width=1.3cm]{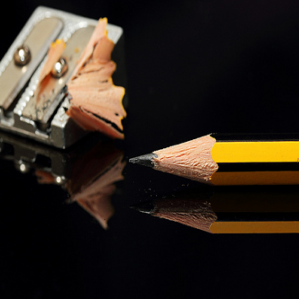}
 \includegraphics[width=1.3cm]{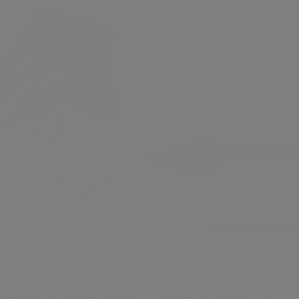} \includegraphics[width=1.3cm]{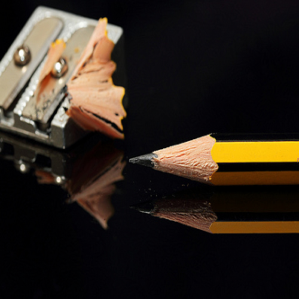}
\end{tabular}}
\caption{\label{fig:attcifar} An illustration of attacked images for CIFAR10 (left) and ImageNet2012 (right).
Each package of three images contains from left to right: input image, attack and the resulting adversarial image.}
\end{figure}

\section{Detection by Counter Attack} \label{sec:optdec}

Let $C=\{1,\ldots,c\}$ denote the set of $2 \leq c \in \mathbb{N}$ distinct classes and let $I=[0,1]$ denote the closed unit interval. An image is an element $x \in I^n$. Let $\varphi: I^n \to I^c$ be a continuous function given by a ReLU-DNN, almost everywhere differentiable, that maps $x$ to a probability distribution $y = \varphi(x) \in I^c$ with $\sum_{i=1}^c y_i = 1$ and $y_i \geq 0$.
Note that this includes convolutional layers (among others) as they can be viewed as fully connected layers with weight sharing and sparsity.
Furthermore, $\kappa: I^n \to C$ denotes the map that yields the corresponding class index, i.e.,
\begin{equation}
\kappa(x) = \begin{cases}
               i & \text{if } y_i > y_j \; \forall j \in C \setminus\{i\} \\
               0 & \text{else.}
           \end{cases}
\end{equation}
This is a slight modification of the $\argmax$ function as $K_i = \{ x \in I^n \, : \, \kappa(x) = i \}$ for $i>0$ gives the set of all $x \in I^n$ predicted by $\varphi$ to be a member of class $i$ and for $i=0$ we obtain the set of all class boundaries with respect to $\varphi$. Given $x \in \mathbb{R}^{n}$, the $\ell_p$ ``norm'' is defined as follows: for $p \in \mathbb{N}$, $\norm{x}_p = \left( \sum_{i=1}^{n} |x_{i}|^{p} \right)^{\frac{1}{p}} $,
for $p = 0$, $\norm{x}_0 = |\{ x_i > 0 \}|$,
and for $p=\infty$, $\norm{x}_\infty = \max_i |x_i| $.

The corresponding $\ell_p$ distance measure is given by $\mathit{dist}_p(x,x') = \norm{x-x'}_p$ and the $n$-dimensional open $\ell_p$-neighborhood with radius $\varepsilon$ and center point $x_0 \in \mathbb{R}^n$ is $B_p(x_0,\varepsilon) = \{ x \in \mathbb{R}^n \, : \, \mathit{dist}_p(x,x_0) < \varepsilon \}$.

\subsection{The CW Attack}
For any image $x_0 \in I^n$ the CW attack introduced in \cite{DBLP:journals/corr/CarliniW16a} can be formulated as the following optimization problem:
\begin{equation} \label{eq:CWopt}
\begin{aligned}
    x_k =\ &\underset{x \in \mathbb{R}^{n}}{\argmin} \; \mathit{dist}_p(x_0,x) \\
    &\text{s.th.} ~~ \kappa(x) \neq \kappa(x_0) \; x \in I^n.
\end{aligned}
\end{equation}
Several reformulations of \cref{eq:CWopt} are proposed in \cite{DBLP:journals/corr/CarliniW16a}. The condition $\kappa(x) \neq \kappa(x_0)$ is replaced by a differentiable function $f(x)$. The reformulation we adopt here is based on the penalty approach. The objective in \cref{eq:CWopt} is replaced by $\mathit{dist}_p(x_0,x)^p$ for $p \in\mathbb{N}$ and left unchanged for $p=0,\infty$. Therefore, the problem we consider from now on is 
\begin{equation} \label{eq:relaxed}
F(x) := \mathit{dist}_p(x_0,x)^p + a \, f(x) \to \min, x \in I^n
\end{equation}
for the penalty parameter $a$ large enough.
In their experiments, Carlini \& Wagner perform a binary search for the constant $a$ such that after the final iteration $\kappa(x) \neq \kappa(x_0)$ is almost always satisfied. 
For further algorithmic details we refer to \cite{DBLP:journals/corr/CarliniW16a}. Two illustrations of attacked images are given in \cref{fig:attcifar}. In particular for the ImageNet dataset even the perturbations themselves are imperceptible.

\subsection{The CW Counter Attack}

The goal of the CW attack is to
find a minimizer of the distance of $x_0$ to the closest class boundary via $\mathit{dist}_p(x_0,x_k)\to\min$ under the constraint that $x_k$ lies beyond the boundary. The more eager the attacker minimizes the perturbation $x_k - x_0$ in a chosen $p$-norm, the more likely it is that $x_k$ is close to a class boundary. Hence, when estimating the distance of $x_k$ to the closest class boundary by performing another CW attack 
$x_{k,j}$ that is started at $x_k$, it is 
to be expected that
\begin{equation}
    \mathit{dist}_p(x_k,x_{k,j}) \ll \mathit{dist}_p(x_0,x_k) \, ,
\end{equation}
for $j$ sufficiently large, cf.~\cref{fig:repeatedAttack}. This motivates our claim, that the CW attack itself is a good detector for CW attacks. A non-attacked image $x_0 \in I^n$ is likely to have a greater distance to the closest class boundary than an $x_k \in I^n$ which has already been exposed to a CW attack.
In practice, it cannot be guaranteed that the CW attack finds a point $x_k$ which is close to the closest boundary. However, we can guarantee to find a stationary point of the problem in \cref{eq:relaxed} that in case of a successful attack ($\kappa(x)\neq\kappa(x_0)$) is guaranteed to lie on the decision boundary. 

\def\secondSituation{
\draw[dashed, color={rgb, 255:red, 208; green, 104; blue, 27 }, line width=1.5, opacity=0.4 ] (340,180) ellipse (80 and 80);
\draw[fill, color={rgb, 255:red, 208; green, 104; blue, 27 }, opacity=0.2 ] (340,180) ellipse (80 and 80);
\draw[<->, color={rgb, 255:red, 208; green, 104; blue, 27 }] (340,180) -- (340,100) node[right,midway] {$\varepsilon$};
\draw  [color={rgb, 255:red, 208; green, 104; blue, 27 } , opacity=0.4 ][line width=1.5]  (330,180) -- (350,180)(340,170) -- (340,190) ;
\draw   (50.49,-7.52) .. controls (281.49,12.17) and (338.39,285.18) .. (569.39,304.87) ;
\draw  [color={rgb, 255:red, 208; green, 2; blue, 27 }  ,draw opacity=1 ][line width=1.5]  (467.54,63.5) -- (488.12,63.5)(477.83,54) -- (477.83,73) ;
\draw  [color={rgb, 255:red, 208; green, 2; blue, 27 }  ,draw opacity=1 ][line width=1.5]  (308.62,200.5) -- (329.2,200.5)(318.91,191) -- (318.91,210) ;
\draw [color={rgb, 255:red, 144; green, 19; blue, 254 }  ,draw opacity=1 ][line width=1.5]  [dash pattern={on 5.63pt off 4.5pt}]  (477.59,65) .. controls (493.7,127.37) and (347.49,124.07) .. (325.06,193.86) ;
\draw [shift={(324.42,196)}, rotate = 285.6] [color={rgb, 255:red, 144; green, 19; blue, 254 }  ,draw opacity=1 ][line width=1.5]    (14.21,-4.28) .. controls (9.04,-1.82) and (4.3,-0.39) .. (0,0) .. controls (4.3,0.39) and (9.04,1.82) .. (14.21,4.28)   ;

\draw [color={rgb, 255:red, 80; green, 166; blue, 35 }  ,draw opacity=1 ][line width=1.5]  [dash pattern={on 5.63pt off 4.5pt}]  (311.97,193) .. controls (279.14,171.44) and (274.8,123.95) .. (316.09,122.06) ;
\draw [shift={(318.67,122)}, rotate = 180] [color={rgb, 255:red, 80; green, 166; blue, 35 }  ,draw opacity=1 ][line width=1.5]    (14.21,-4.28) .. controls (9.04,-1.82) and (4.3,-0.39) .. (0,0) .. controls (4.3,0.39) and (9.04,1.82) .. (14.21,4.28)   ;

\draw  [color={rgb, 255:red, 208; green, 2; blue, 27 }  ,draw opacity=1 ][line width=1.5]  (317.24,124.5) -- (337.82,124.5)(327.53,115) -- (327.53,134) ;
\draw[<->, color={rgb, 255:red, 0; green, 0; blue, 0 }] (340,180) -- (477.83,63.5) node[left, near end] {$\delta$};

\draw (499.92,45) node [scale=1] [align=left] {$x_0$};
\draw (109.54,148) node [scale=1] [align=left] {$\NF$};
\draw (368.93,34) node [scale=1] [align=left] {$\F$};
\draw (433.08,145) node [color={rgb, 255:red, 144; green, 19; blue, 254 }  ,opacity=1 ] [align=left] {attack no.\ 1};
\draw (323.94,228) node [scale=1] [align=left] {$x_\textit{adv}$};
\draw (291.04,98) node [scale=1] [align=left] {$x_\textit{adv2}$};
\draw (238.18,155) node [color={rgb, 255:red, 80; green, 166; blue, 35 }  ,opacity=1 ] [align=left] {attack no.\ 2};
}

\subsection{Theoretical Considerations} \label{sec:theocons}

We now turn to the proof of what has been announced in \cref{sec:ourcontr}.
%
Let $t:= \kappa(x_0)$ denote the original class. We assume that $p=2$ and fix a choice for $f$ which is
\begin{equation} \label{eq:penality1}
    f(x) =  \max\{Z_t(x) - \max_{i\neq t} \{ Z_i(x) \}, \, 0 \} \, ,  
\end{equation}
where $Z$ denotes the neural network $\varphi$, but without the final softmax activation function. Note that $f$ in \cref{eq:penality1} is a construction for an untargeted attack. Choosing a penalty term for a targeted attack does not affect the arguments provided in this section, the convergence on the CW attack to a stationary point only requires assumptions on the geometry of $\mathrm{Im}(f)$ that are fulfilled for both the targeted and the untargeted attacked.
Furthermore, let 
$\F := \{ x \in I^n \, : \, \kappa(x) \neq t \} $ the feasible region of \eqref{eq:relaxed}, $\NF := \{ x \in I^n \, : \, \kappa(x) = t \} $ the infeasible region, and $\partial \F = \partial \NF$ the class boundary between $\F$ and $\NF$. It is common knowledge that for a ReLU-DNN $I^n$ can be decomposed into a finite number of polytopes $\pol := \{ Q_j \, : \, j=1,\ldots,s \}$ such that $Z$ is affine on each $Q_j$, 
see \cite{HeinPolytopes,UnderstandingML}. In \cite{HeinPolytopes}, each of the polytopes $Q_j$ is expressed by a set of linear constraints, therefore being \emph{convex}. 
For a locally Lipschitz function $g$, the \emph{Clarke generalized gradient} is defined by
\begin{equation}
    \nabla_C g(x) := \mathrm{conv} \{ \lim_{i \to \infty} \nabla g(x_i) : x_i \to x \wedge \exists \nabla  g(x_i) \} \, ,
\end{equation}
where $\mathrm{conv}$ stands for the convex hull of a set of vectors. Let $ S $ be the stationary set of problem \eqref{eq:CWopt},
\begin{align}
   S =  \{ x \in I^n: 0 \in  \nabla_C F(y) + N_{I^n}(x)\} 
\end{align}
where $ \nabla_C F(x) $ is the set of all generalized gradients and $ N_{I^n}(x) = \{ z \in \mathbb{R}^n : z^T(z'-x) \leq 0 \; \forall z' \in I^n \} $ is the normal cone to the set $ I^n $ at point $x$. Note that $N_{I^n}(x) = \emptyset$ for $x \in (0,1)^n$. 
In general, computing the set of generalized gradients is not an easy task, but given the special structure of $ f $ -- piece-wise linear -- it can be done relatively easily in this case. Namely, the set $ \nabla_C f(x) $ is a convex hull of the gradients of linear functions that are active at the point $x$, \cite{Sch12}. Therefore, by \cite{Clarke} (Corollary2, p.39), the generalized gradient  $ G(x) $ of $ F(x) $ has the form 
\begin{equation}
G(x) = \{ a g(x) + 2(x-x_0), \; g(x) \in \nabla_C f(x) \} \, .    
\end{equation}
The projected generalized gradient method \cite{Solodov} is defined as follows. Let $ P(x) $ denote the orthogonal projection of $x $ onto $ I^n. $ Given a non increasing learning schedule $ \{\alpha_k\} $ such that
\begin{equation} \label{learning}  \lim_{k \to \infty} \alpha_k = 0 \mbox{ and } \sum_{k=1}^{\infty} \alpha_k = \infty, 
\end{equation}  the iterative sequence is generated as 
\begin{equation} \label{eq:pg}
x_{k+1} = P(x_k - \alpha_k G_k) , \; G_k \in \nabla_C F(x^k) 
\end{equation}
for $k=0,1,\ldots$. In this paper we add the condition $\sum_{k=1}^{\infty} \alpha^2_k < \infty \, , $
which strengthens the conditions \eqref{learning} and does not alter the statements from \cite{Solodov}. 

\begin{theorem} \label{theorem:stationarypoint}
Under the assumptions outlined above, every accumulation point of the sequence $ \{x_k\} $ in $(0,1)^n$ generated by (\ref{eq:pg}) converges to a stationary point of $F$.
\end{theorem}

The proof of this theorem can be found in \cite{Solodov} and only requires that the set $S$ only contains isolated points from $I^n$, which we prove in \cref{sec:isolated}.

Now, considering an iterate $x_k \in \F$, it is also clear that reducing $F(x_k)$ means decreasing the distance of $x_k$ to $x_0$, provided $\alpha_{k+1}$ is sufficiently small. Therefore, $x_k \in \F$ cannot be a stationary point. Assume that the CW attack converges to $x^*$ and is successful, i.e., $x^* \notin \NF$. This implies $x^* \in \partial \F$. Summarizing this, we obtain:

\begin{lemma}
Any successful CW attack converges to a stationary point on the boundary $\partial \F$ of the feasible set.
\end{lemma}
%
Let us now consider the counter attack and therefore let $\Ftwo \subset \NFone := \NF$ denote the counter attack's feasible region and $\NFtwo \subset \Fone := \F$ the infeasible region.
We seek to minimize the functional
\begin{equation}
    F^{(2)}(x) := \dist_p(x_k,x)^p + b f^{(2)}(x)
\end{equation}
where $x^k$ final iterate of the primary attack and $b$ is suitably chosen. The penalty $f^{(2)}$ is chosen as in \cref{eq:penality1} but with $t=\kappa(x_k)$. Let $\{ Q_1,\ldots,Q_s \}$ be the set of all polytopes that fulfill $x^* \in Q_i, \; i=1,\ldots,s$. We assume that the first attack is successful and has iterated long enough such that the final iterate $x_k$ of the first attack is so close to its stationary point that $\dist(x^*,x_k) < \varepsilon$ and $B(x^*,3\varepsilon) \subset \bigcup_{i=1}^s Q_i \subseteq I^n $. 

\begin{theorem} \label{theorem:boundedness}
Consider the preceding assumptions and further assume that $f$ has no zero gradients inside all polytopes, there are no polytope boundaries at decision boundaries, the counter attack starts at $x_k \in \NFtwo$, stops when reaching $\Ftwo$, uses a sufficiently small initial step size $\alpha_0$ and a sufficiently large $b$. 
Then the counter attack iteration minimizing $F^{(2)}$ never leaves $B(x^*,3\varepsilon)$.
\end{theorem}

Note that the first two assumptions should almost always hold true  for any DNN that obtains its weights by ordinary gradient descent.

%
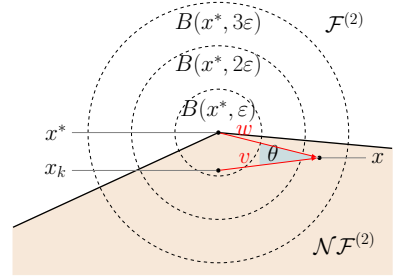
\begin{figure}[tb]
    \centering
    \scalebox{0.42}{
\begin{tikzpicture}[thick,font=\sffamily\huge]
\draw[fill=black!10!white,color=black!20!orange!15!white] (0.5,-3) --  (-6,-6) -- (-6,-7.5) -- (6,-7.5) -- (6,-3.5) -- cycle;
\node[circle,draw,black,dashed,minimum width=2.75cm] at (0.5,-3) {};
\node[circle,draw,black,dashed,minimum width=5.5cm] at (0.5,-3) {};
\node[circle,draw,black,dashed,minimum width=8.25cm] at (0.5,-3) {};

\draw[very thick] (0.5,-3) --  (-6,-6);
\draw[very thick] (0.5,-3) --  (6,-3.5);

\node [fill, draw, circle, minimum width=3pt, inner sep=0pt, pin={[outer sep=1.5pt, pin distance=130pt]180:$x_k$}] (xk) at (0.5,-4.2) {};
\node [fill, draw, circle, minimum width=3pt, inner sep=0pt, pin={[outer sep=1.5pt, pin distance=130pt]180:$x^*$}] (xs) at (0.5,-3) {};
\node [fill, draw, circle, minimum width=3pt, inner sep=0pt, pin={[outer sep=1.5pt, pin distance=40pt]0:$x$}] (x) at (3.7,-3.8) {};

\draw pic[<->, fill=\mycolor!50!white,
      angle radius=19mm, angle eccentricity=1.2, 
      pic text = $\theta$, pic text options={right=5mm}] {angle = xs--x--xk};
      
\draw[thick,-latex,color=red] (xs) --  (x) node[near start,above] {$w$};
\draw[thick,-latex,color=red] (xk) --  (x) node[near start,above] {$v$};

\node at (4.5,-6.5) {$\NF^{(2)}$};
\node at (4.5,0.5) {$\F^{(2)}$};
\node at (0.5,-2.3) {$B(x^*,\varepsilon)$};
\node at (0.5,-0.875) {$B(x^*,2\varepsilon)$};
\node at (0.5,0.45) {$B(x^*,3\varepsilon)$};
\end{tikzpicture}
}
\caption{Situation present in the counter attack. By bounding $\cos(\theta)$ from below we show that $x$ reduces its distance to $x^*$ when performing a gradient descent step under the given assumptions. Therefore, $x$ never leaves $B(x^*,3\varepsilon).$}
\label{fig:counterattack}
\end{figure}
\begin{proof} Let $x \in B(x^*,3\varepsilon) \setminus B(x^*,2\varepsilon)$ and $G^{(2)}(x) \in \nabla_C F^{(2)}(x)$ any gradient with corresponding $g^{(2)}(x) \in \nabla_C f^{(2)}(x)$ which is constant on a given polytope $Q_i$ (containing $x$). We can assume that $x \in \NFtwo{}^\circ$ since we are done otherwise. Furthermore, let $v :=x_k - x$, $w := x^* - x$ and $\cos(\theta) = \frac{v^T w}{\norm{v} \norm{w}}$ where $\norm{\cdot} = \norm{\cdot}_2$ (cf.\ \cref{fig:counterattack}). This implies $x^* - x_k = w - v$ and
\begin{align}
    \cos(\theta) & = \frac{\norm{v}^2 + \norm{w}^2 - \norm{w-v}^2}{2 \norm{v} \norm{w}} \nonumber  \\ 
    & \geq \frac{4\varepsilon^2 + \varepsilon^2 - \varepsilon^2}{2 \cdot 3\varepsilon \cdot 4 \varepsilon} = \frac{1}{6} \, .
\end{align}
Since $G^{(2)}(x) = 2(x-x_k) + b \, g^{(2)}(x)$, we obtain
\begin{align}
         & G^{(2)}(x)^T (x-x^*) \nonumber \\
    = \; & 2(x-x_k)^T (x-x^*) + b \underbrace{ g^{(2)}(x)^T (x-x^*) }_{>0} \nonumber \\[-1.5ex]
    > \; & 2(x-x_k)^T (x-x^*) 
    \geq \; 4 \varepsilon^2 \cos(\theta) = \frac{2\varepsilon^2}{3} \, .
\end{align}
Note that, $g^{(2)}(x)^T (x-x^*) >0$ follows from $x-x^*$ being an ascent direction. This holds due to $x \in Q_i \cap \NFtwo{}^\circ$ which implies $f^{(2)}(x)>0$ and the fact that $x^* \in Q_i$ as well as $f^{(2)}(x^*)=0$.
Let $ \mathcal{B} := \bigcup \{ Q_i : Q_i \cap \partial\Ftwo{} \neq \emptyset \} $ and
\begin{equation}
    c := \min_{z \in \mathcal{B} } \| g^{(2)}(z) \|_2>0, \quad  C:= \max_{z \in \mathcal{B} } \| g^{(2)}(z) \|_2 \, .
\end{equation}
If $b$ is sufficiently large, any stationary point $x$ fulfills $2(x-x_k) = b \sum \lambda_i g_i^{(2)}(x)$ with $\lambda_i \geq 0$ and $\sum \lambda_i \leq 1$ where the sums are w.r.t.\ all polytopes $Q_i$ that contain $x$. 
This requires $ \|x-x_k \| < 8\varepsilon \leq bc $, i.e., $b = \frac{8\varepsilon}{c}$ is sufficiently large.
Therefore $\norm{G^{(2)}(x)}^2$ can be bounded from above as follows:
\begin{align}
    & 4 \norm{x-x_k}^2 + 4b (x-x_k)^T g^{(2)}(x) + b^2 \norm{ g^{(2)}(x) }^2  \nonumber \\ 
    \leq & 4(4\varepsilon)^2 + 4 \frac{8\varepsilon}{c} 4\varepsilon C + \frac{64 \varepsilon^2}{c^2} C^2 = 64\varepsilon^2 \! \left(1+\frac{C}{c}\right)^2 \!\!\!  .
\end{align}
Hence, the difference in distance to $x^*$ when performing a gradient descent step is
\begin{align} \label{eq:defcC}
    & \norm{x - x^*}^2 - \norm{x - x^* - \alpha_k G^{(2)}(x) }^2 \nonumber \\
    = \; & 2\alpha_k G^{(2)}(x)^T ( x - x^* ) - \alpha_k^2 \norm{G^{(2)}(x)}^2 \nonumber \\
    \geq \; & 2\alpha_k \cdot \frac{2\varepsilon^2}{3} - \alpha_k^2 64\varepsilon^2\left(1+\frac{C}{c}\right)^2 \, .
\end{align}
The latter expression is greater than zero if the step-size schedule $\alpha$ is small enough, i.e., $\alpha_k < \frac{1}{16(1*\frac{C}{c})^2}$. Hence, the distance to $x^*$ decreases. 
Note that $\alpha_k$ is hence small enough such that $x + \alpha_k G^{(2)}(x) \in (0,1)^n$ and we can omit the projection onto $I^n$. A second requirement is $\alpha \norm{G^{(2)}(x)} < \varepsilon$ such that $B(x^*,3\varepsilon) \setminus B(x^*,2\varepsilon)$ cannot be skipped by any $x \in B(x^*,2\varepsilon)$ within a single gradient descent step, however a simple derivation shows that this condition leads to a weaker bound on $\alpha$. This concludes the proof.
\end{proof}



\newcommand{\DRV}[1]{D(#1)}
\newcommand{\Dn}[0]{D^{(0,k)}}
\newcommand{\Dnm}[1]{\bar{D}^{(k,#1)}}

Note that the preceeding theorem and proof do not depend on whether we choose the first attack to be targeted or untargeted.
We now take into account the stochastic effects that stem from choosing an arbitrary initial image $x \in I^n$ represented by a random variable $X$. Let $k$ be the number of iterates of the original CW attack and let $X_{k} \in \F^{(1)}$ be the final iterate after $k \geq 1$ iterations, starting at $X_{0} = X$. For any random variables $Y,Z$ with values in $I^n$ let
\begin{align}
    \DRV{Y} & = \dist(Y,\F) = \inf_{y \in \F} \dist(Y,y) \quad \text{and} \nonumber\\
    \Dn & = \dist(X_{0},X_{k}) \geq \DRV{X_{0}} \, ,
\end{align}
provided $X_k\in\F^{(1)}$, which means that the $k$-th iterate is a successful attack.
For $\tau \in \mathbb{R}$, we define the cumulative density function corresponding to a random variable $D$ (representing a random distance) by
    $F_D(\tau) = P( D \leq \tau ) = D_*P( (-\infty,\tau])$
where the push forward measure is $D_* P( B ) = P(D^{-1}(B))$ for all $B$ in
the Borel $\sigma$-algebra. The area under receiver operator characteristic curve of $\DRV{Y}$ and $\DRV{Z}$ is defined by
\begin{align}
    \mathrm{AUROC}\!\left(\DRV{Y},\DRV{Z}\right)
    & = \int_{\mathbb{R}_+} \! F_{\DRV{Y}} \!( \tau ) \, d\DRV{Z}_*P ( \tau ). 
\end{align}
For a derivation of the latter definition, see \cref{sec:continuousAUROC}.
Obviously, it holds that $\DRV{Y},\Dn \geq 0$. The following lemma formalizes under realistic assumptions that we obtain perfect separability of $\DRV{X}$ and $\DRV{X_{k}}$ as we keep iterating the initial CW attack, i.e., $k \to \infty$.

\begin{lemma}
\label{lem:AUROC}
Let $\DRV{X} \geq 0$ with $P(\DRV{X} = 0) = 0$ and $\DRV{X_{k}} \to 0$ for $k \to \infty$ weakly by law. Then,
\begin{equation}
    \mathrm{AUROC}\!\left(\DRV{X_{k}},\DRV{X}\right) \stackrel{k \to \infty}{\longrightarrow} 1 \, .
\end{equation}
\end{lemma}
\begin{proof}
Let $\delta_0$ be the Dirac measure in $0$ with distribution function $F_{\delta_0}(z) = \{ 1 \text{ for } z \geq 0, \; 0 \text{ else} \}$.
By the characterization of weak convergence in law by the Helly-Bray lemma 
\cite{Ferguson},
$F_{\DRV{X_{k}}}(\tau) \to F_{\delta_0}(\tau)$ for all $\tau$ where $F_{\delta_0}(\tau)$ is continuous. This is the case for all $\tau \neq 0$. As $P(\DRV{X} = 0) = 0$, this implies $F_{\DRV{X_{k}}} \to F_{\delta_0}$ $\DRV{X}_* P$-almost surely. Furthermore, it holds that $|F_{\DRV{X_{k}}}(\tau)| \leq 1$. Hence, by Lebesgue's theorem of dominated convergence
\begin{align}
    \mathrm{AUROC}\!\left(\DRV{X_{k}},\DRV{X}\right) = & \; \int_{\mathbb{R}} F_{\DRV{X_{k}}}(\tau) \, d \DRV{X}_* P( \tau ) \nonumber \\
    \stackrel{k \to \infty}{\longrightarrow} & \; \int_{\mathbb{R}} F_{\delta_0}(\tau) \, d \DRV{X}_* P( \tau ) \nonumber \\
    = & \; \int_{\mathbb{R}_+} \hspace{-1ex} d \DRV{X}_* P( \tau ) = 1 
\end{align}
since $\DRV{X} \geq 0$ by assumption.
\end{proof}

Now, let $X$ be a random input image with a continuous density. Assuming that decision boundaries of the neural net have Lebesgue measure zero, $X \in \NF^{(1)}$ holds almost surely,
and therefore $\DRV{X} > 0$ almost surely, such that indeed $P(\DRV{X} = 0) = 0$. Furthermore, let $X_{k}$ be the $k$-th iterate inside $\F^{(1)}$ of the CW attack starting with $X$. 
Conditioned on the event that the attack is successful, we obtain that $X_{k} \stackrel{k \to \infty}{\longrightarrow} X^* \in \partial \F$, thus
\begin{equation}
    \dist(X_{k},X^*) \stackrel{k \to \infty}{\longrightarrow} 0
\end{equation}
almost surely for the conditional probability measure, which we henceforth use as the underlying measure. Now there exists a learning rate schedule $\alpha_{k,j}$ and a penalty parameter $a$ such that for all steps $X_{k,j}$ of the counter attack originating at $X_{k}$, the distance
\begin{equation}
    \Dnm{j} = \dist(X_{k},X_{k,j}) \leq 8 \, \dist(X_{k},X^*) \stackrel{k \to \infty}{\longrightarrow} 0
\end{equation}
almost surely. Note that the latter inequality holds since we can choose $\varepsilon$ in \cref{theorem:boundedness}, such that $\dist(x_k,x^*) > \varepsilon/2$ and then obtain the constant by the triangle inequality. Let $X_{k,j^*}$ be the first iterate of the CW counter attack in $\F^{(2)}$. We consider
\begin{equation}
 \Dnm{j^*} \!\! = \dist(X_{k},X_{k,j^*}) \leq 8 \, \dist(X_{k},X^*) \! \stackrel{k \to \infty}{\longrightarrow} \! 0   
\end{equation}
almost surely. 
Hence, we obtain by application of \cref{lem:AUROC}:

\begin{theorem}
Under the assumptions outlined above, we obtain the perfect separation of the distribution of the distance metric $\Dnm{j^*}$ of the CW counter attack from the distribution on the distance metric $\Dn$ of the original CW attack, i.e.,
\begin{equation}
    \mathrm{AUROC}\left( \Dn, \Dnm{j^*} \right) \stackrel{k \to \infty}{\longrightarrow} 1 \, .
\end{equation}
\end{theorem}

\begin{figure*}[t]
    \centering
    \subfigure{\includegraphics[trim=25 10 45 10,clip,width=0.23\textwidth]{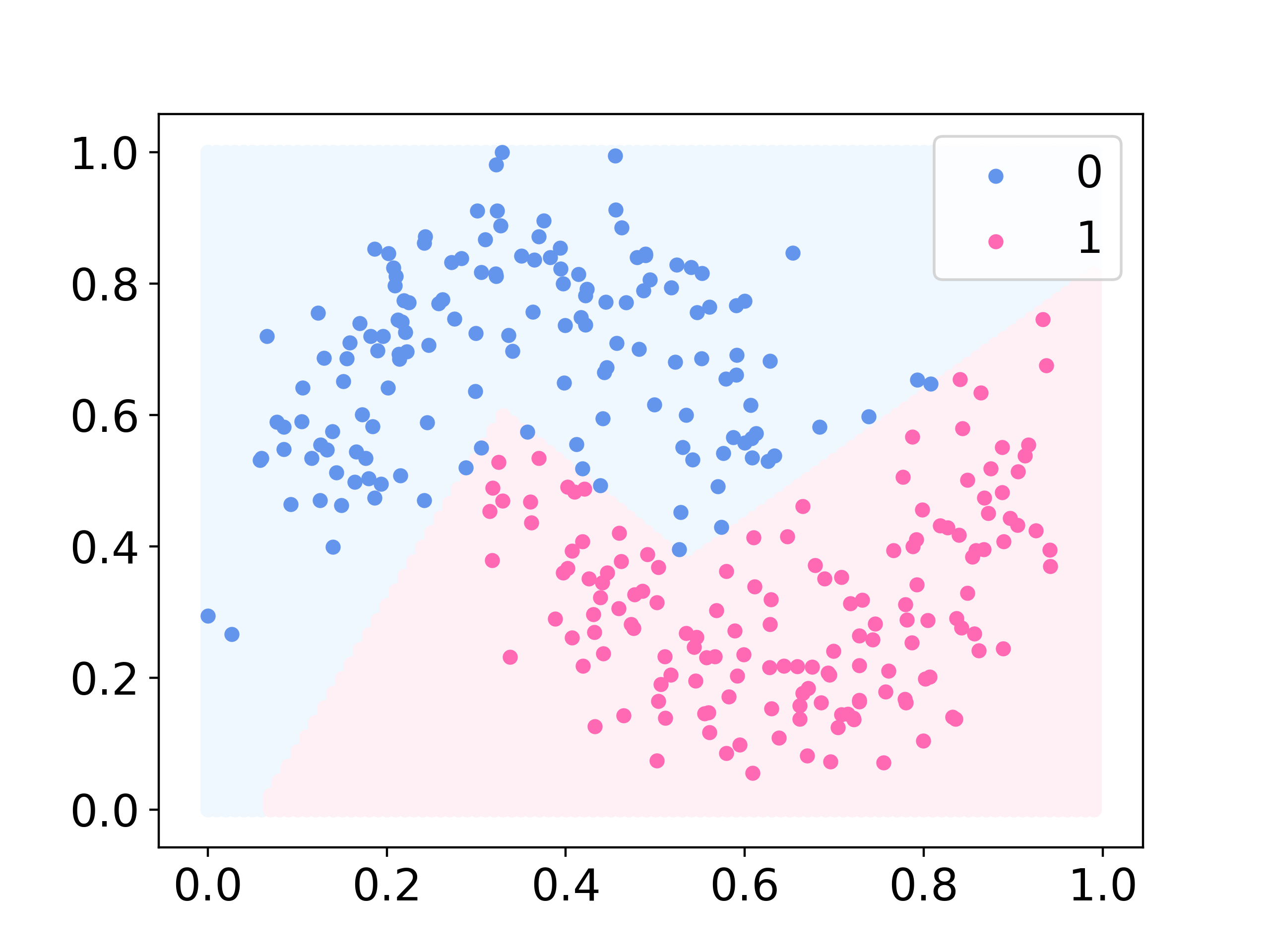}}
    \subfigure{\includegraphics[trim=28 10 65 25,clip,width=0.21\textwidth]{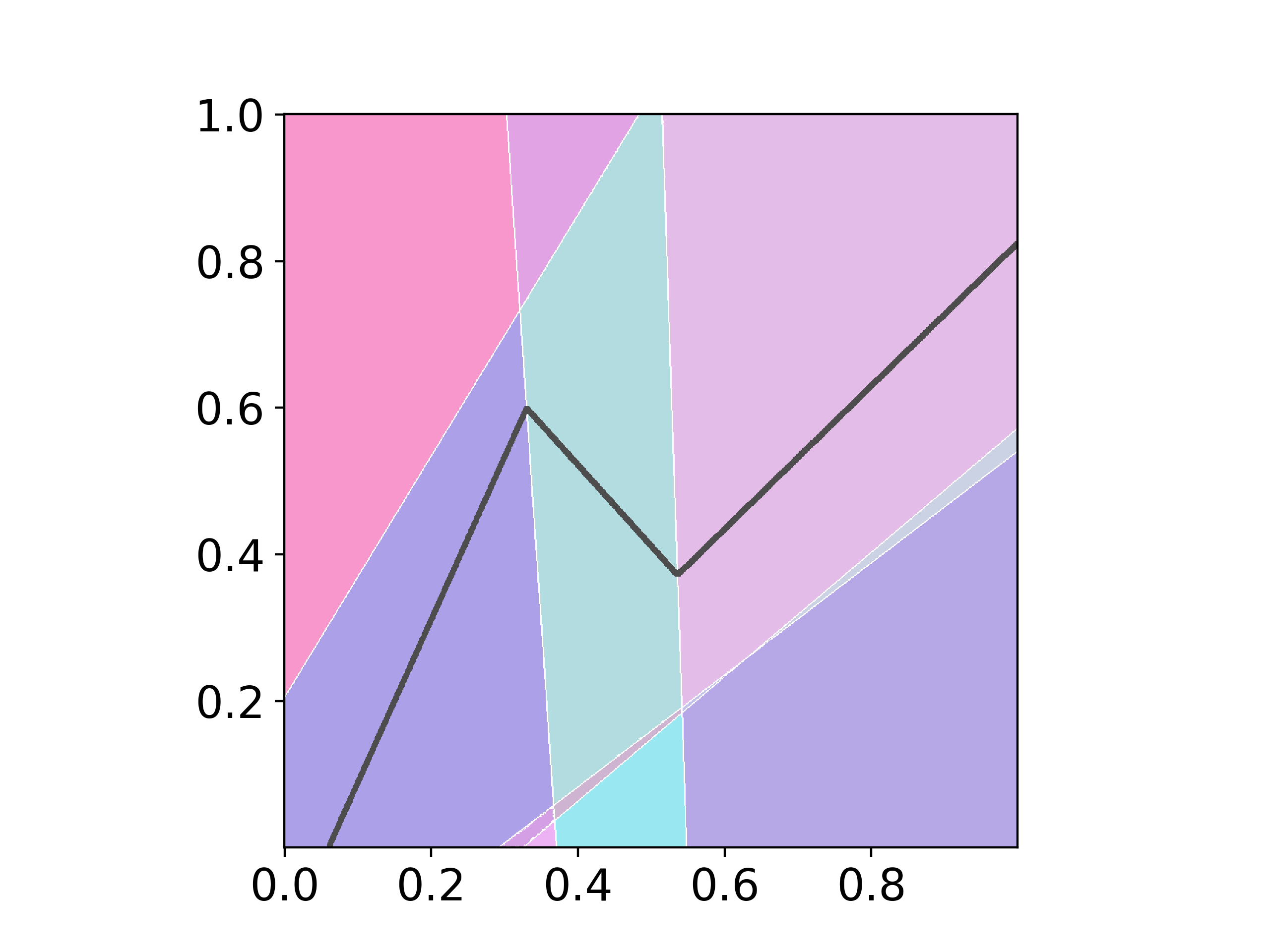}}
    \subfigure{\includegraphics[trim=80 50 50 60,clip,width=0.29\textwidth]{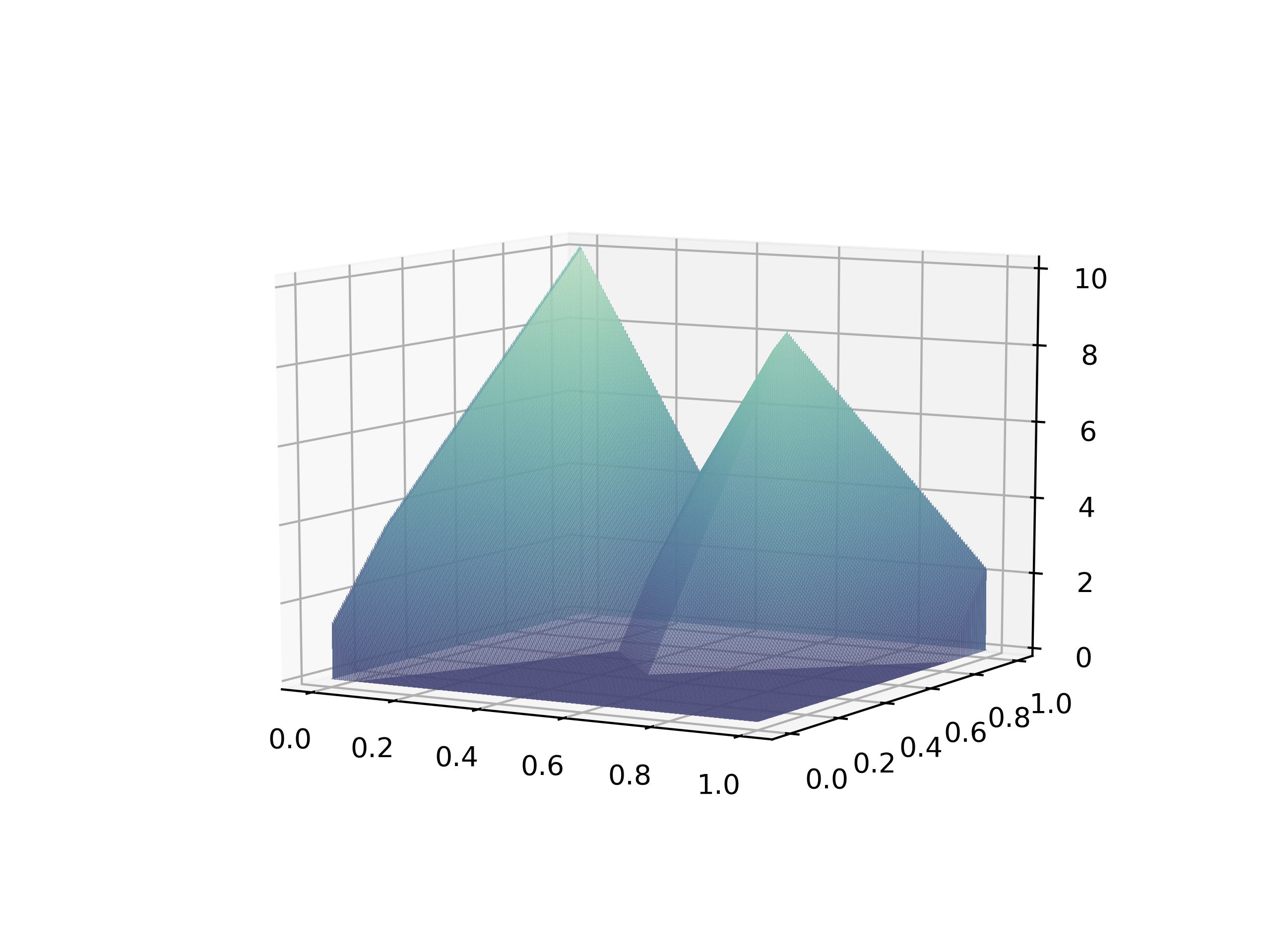}}
    \subfigure{\includegraphics[trim=5 2 5 5,clip,width=0.24\textwidth]{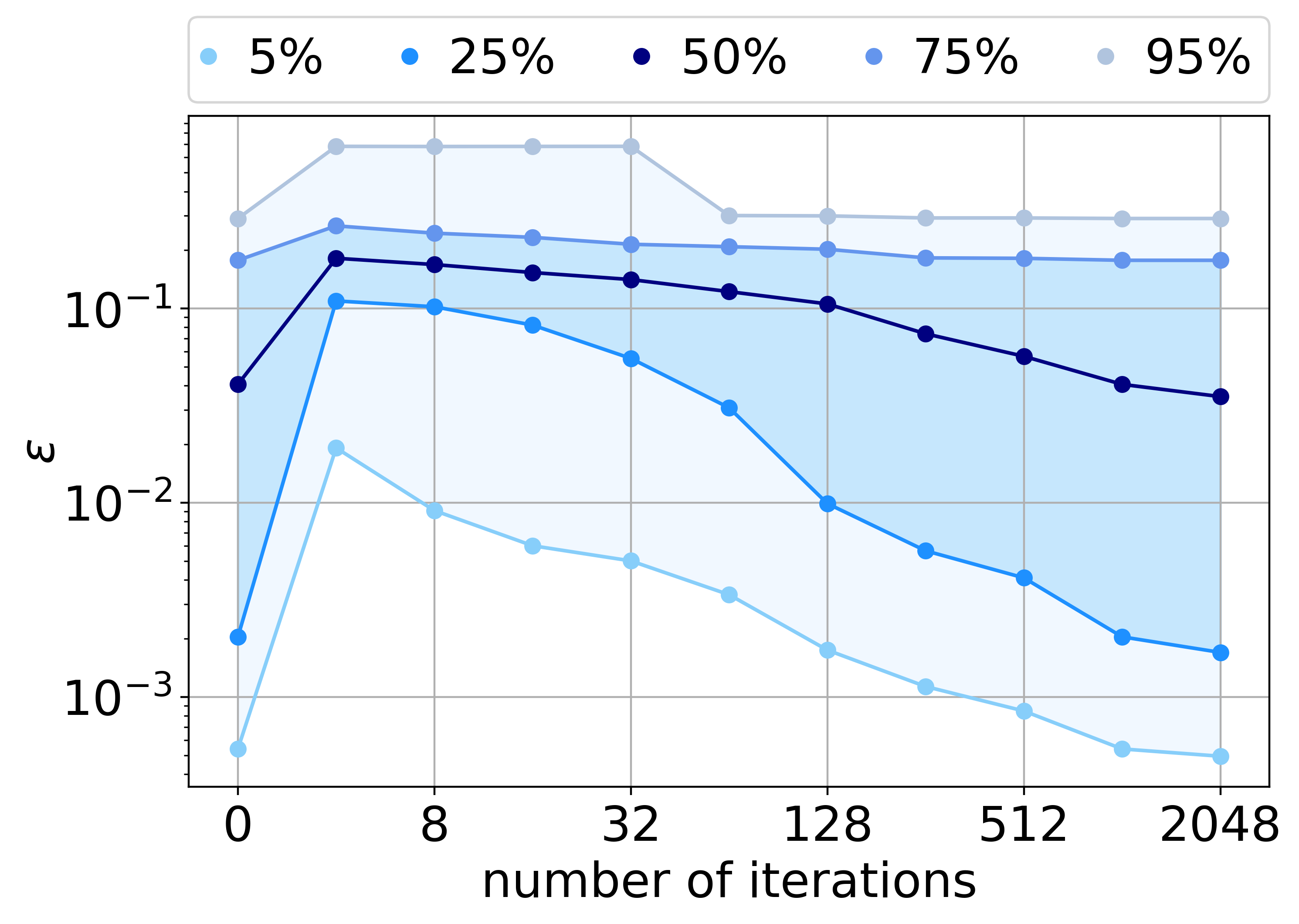}}
    \caption{(left): Classification results for the two moons example of the test dataset. (center left): A decomposition of $I^2$ into a set of polytopes for the used ReLU network with decision boundary (black line). (center right): Surface plot of loss function $f(x)$ (\cref{eq:penality1}) for class $0$. (right): Quantiles for the $\varepsilon$ values (log scaled) as function of the number of iterations.}
    \label{fig:moon_data}
\end{figure*}

\section{Numerical Experiments} \label{sec:numex}

\begin{figure}[t]
    \centering
    \subfigure{\includegraphics[trim=1 1 1 1,clip,width=0.24\textwidth]{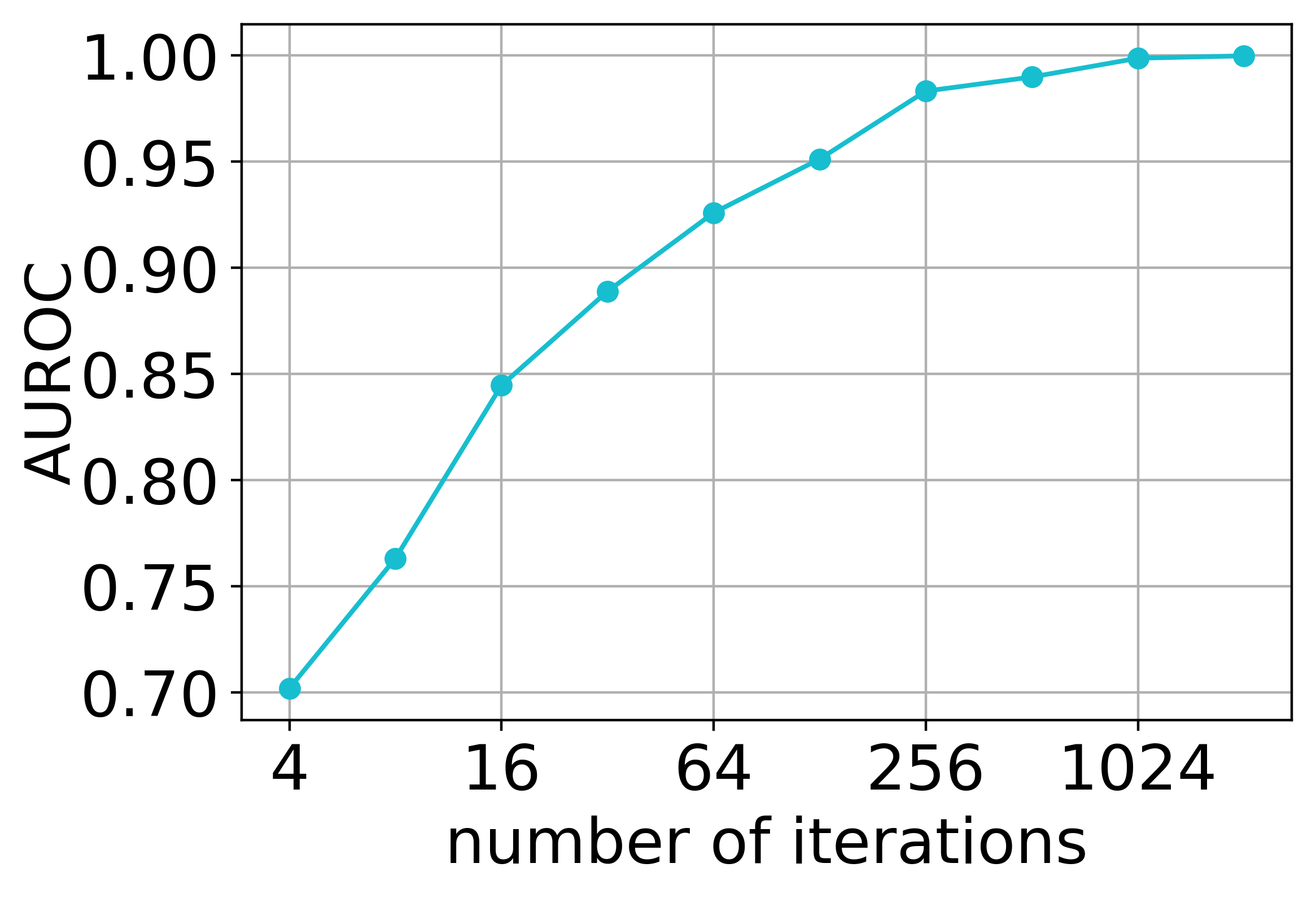}}
    \subfigure{\includegraphics[trim=1 1 1 1,clip,width=0.24\textwidth]{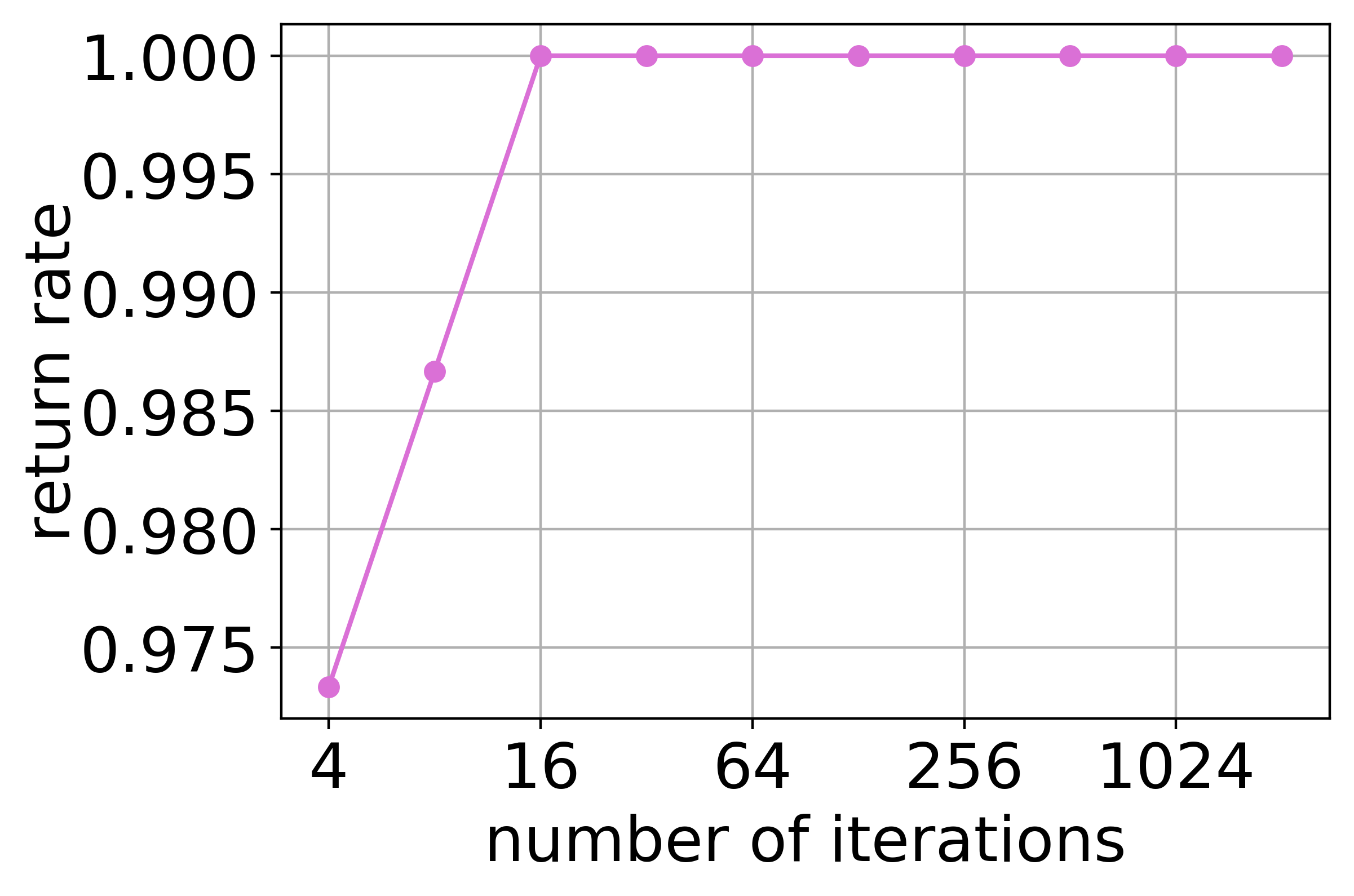}}
    \caption{(left): AUROC values as function of number of iterations for classifying between $D^{(0,j)}$ and $\bar{D}^{(k,j)}$. (right): Return rates vs.\ number of iterations for the two moons example.}
    \label{fig:moon_results}
\end{figure}

We now demonstrate how our theoretical considerations apply to numerical experiments. First, we introduce a two-dimensional classification problem which we construct in such a way that the theoretical assumptions are strictly respected and the mechanisms underlying our proof of perfect separability in the large iteration limit of the primary attack are illustrated. To this end, we create a dataset based on the two moons example. We use $2,\!000$ data points to train a classifier and $300$ data points as test set achieving a test accuracy of $93.33\%$. The test dataset and corresponding class predictions are shown in \cref{fig:moon_data} (left).
As classifier, a shallow ReLU network with one hidden layer consisting of $8$ neurons is considered. The decomposition of this network into a finite set of polytopes $\pol$ is given in \cref{fig:moon_data} (center left). A corresponding surface plot of $f(x)$ is shown in \cref{fig:moon_data} (center right).
The penalization strength of the primary attack and the learning rate schedule are defined by 
\begin{equation}
   a \in \big[ \frac{2 \sqrt{d}}{c}, 100  \big] \quad \text{and} \quad  \alpha_{i} = \alpha_{0} \frac{n_{0}}{n_{0} + i} \, ,
\end{equation} 
respectively. For the definition of $c$ recall \cref{eq:defcC} and for the lower bound on $a$ see \cref{sec:penalizationstrength}. For the input dimension $d=2$, we considered up to $2,\!048$ iterations for the primary attack and adjusted $n_0=100$ such that $\sum_{i=1}^{2048} \alpha_{i} > \sqrt{2}$. We chose $\alpha_{0} = 0.01$ for the primary attack.
For the counter attack, we use the parameters defined in the proof of \cref{theorem:boundedness}:
\begin{equation}
   b = \frac{8 \varepsilon}{c} \quad \text{and} \quad \alpha_k = \frac{1}{16(1+\frac{C}{c})^2} \, .
\end{equation} 
We chose $\varepsilon=|| x_{k} - x^{*} || + 2 \cdot 10^{-4}$ individually for each data point $x_0$ (for the justification of the constant offset $2 \cdot 10^{-4}$, see \cref{sec:twomoondim}). The quantiles of calculated $\varepsilon$ values for the different number of iterations are presented in \cref{fig:moon_data} (right). 
The counter attack performs up to $2,\!048$ iterations (independently of the number of iterations of the primary attack) and stops in iteration $j^*$ when $x_{k,j^{*}} \in \F^{(2)}$ for the first time ($x_{k,j^{*}}$ a realization of $X_{k,j^{*}}$).
For all tests, we randomly split the test set into two equally sized and distinct portions $\mX$ and $\bar{\mX}$. For $\bar{\mX}$ and run a primary attack as well as the counter attack generating values $D^{(0,k)}$ and $\bar{D}^{(k,j^*)}$. For $\mX$, we only run a counter attack and obtain $D^{(0,j^*)}$. We used the CleverHans framework \cite{cleverhans2018} which 
includes the CW attack with $p=2$. The implementation contains a confidence parameter $\eta$, reformulating \cref{eq:penality1} into
$f(x) =  \max\{Z_t(x) - \max_{i\neq t} \{ Z_i(x) \} - \eta, \, 0 \} $.
In accordance to our theory, we consider $\eta=0$, recovering \cref{eq:penality1}.

To empirically validate the statement in \cref{theorem:boundedness}, we construct the corresponding $B(x^*,3\varepsilon)$ and use all data points for which $B(x^*,3\varepsilon) \subset \bigcup_{i=1}^s Q_i$ with $x^* \in Q_i, \; i=1,\ldots,s$ is fulfilled.  
%
%
\Cref{fig:moon_data} (right) suggests that more iterations could be considered since stationarity is not yet achieved. However, the CW attack is already computationally costly for $2,\!048$ iterations. Nevertheless, we observe in \cref{fig:moon_results} (left) that the considered number of iterations is sufficient for perfect detection in terms of AUROC. We also registered return rates, i.e., the percentage of cases with successful primary attacks where we obtain $\kappa(x_0)=\kappa(x_{k,j^*})$. In a two class setting, this is guaranteed for all successful primary attacks being close enough to the decision boundary. \Cref{fig:moon_results} (right) shows that the return rate indeed tends to a 100\% for an increasing number of iterations.
We observe in our tests, independently of the number of iterations of the primary attack, that $x_k \in B(x^*,\varepsilon)$ implies $x_{k,j^{*}} \in B(x^*,3\varepsilon)$ for all the data points. Also when the condition $B(x^*,3\varepsilon) \subset \bigcup_{i=1}^s Q_i$ is violated we stopped when $x_{k,j^*} \in \F^{(2)}$ and still found $x_{k,j^*} \in B(x^*,3\varepsilon)$.
For the same tests on the CIFAR10 dataset \cite{CIFAR10} of tiny $32 \times 32$ rgb images with 10 classes, we observe a similar behavior as in \cref{fig:moon_results}, see \cref{fig:cifar_auroc}, although we are not able to check whether our theoretical assumptions hold. We performed at most $1,\!024$ counter attack iterations due to the computational cost of the CW attack. For $1,\!024$ primary attack iterations, only 5 examples did not return to their original class, 3 of them remained in the primary attack's class, 2 moved to a different one. 

\begin{figure}[t]
    \centering
    \subfigure{\includegraphics[trim=1 1 1 1,clip,width=0.24\textwidth]{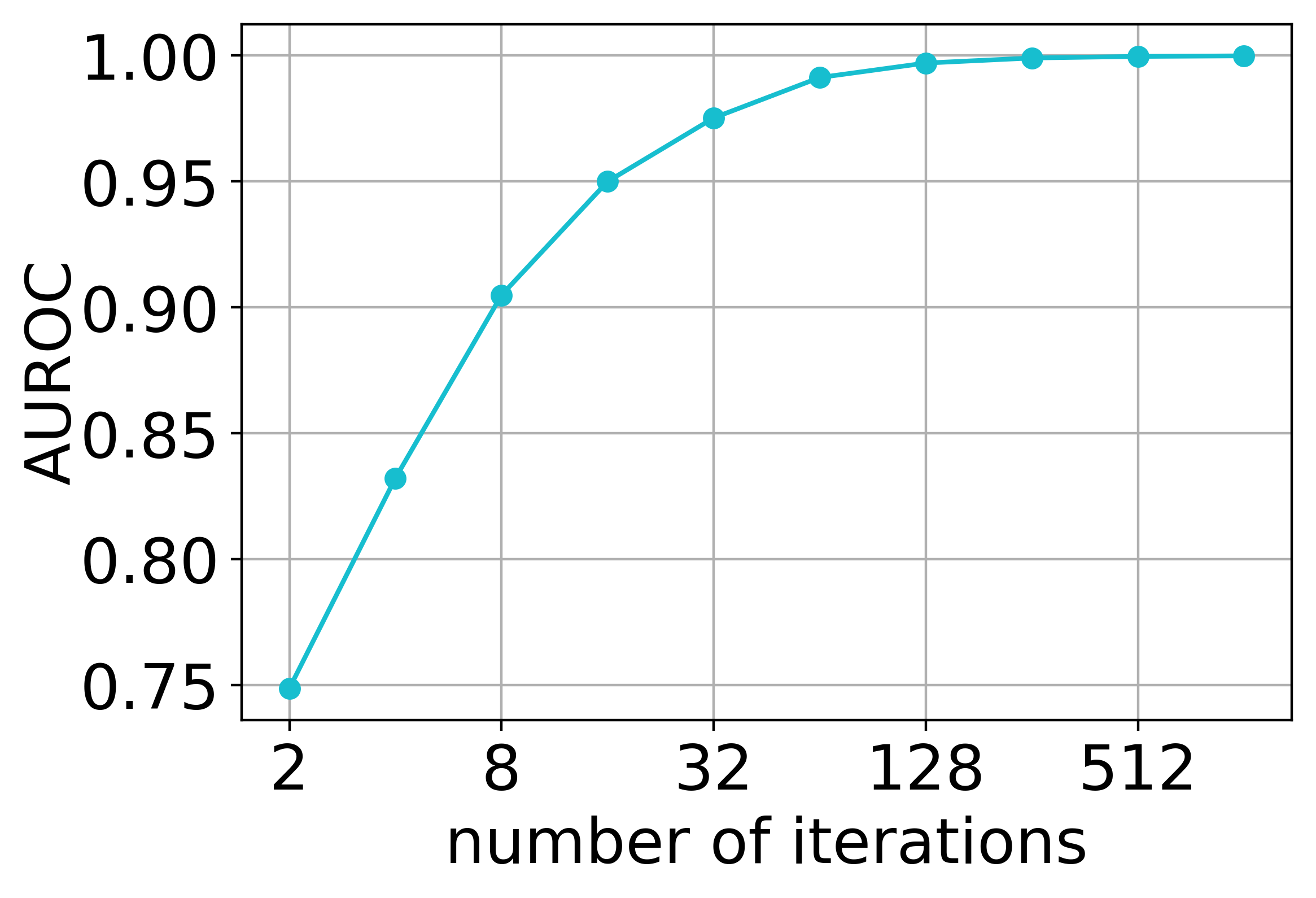}}
    \subfigure{\includegraphics[trim=1 1 1 1,clip,width=0.24\textwidth]{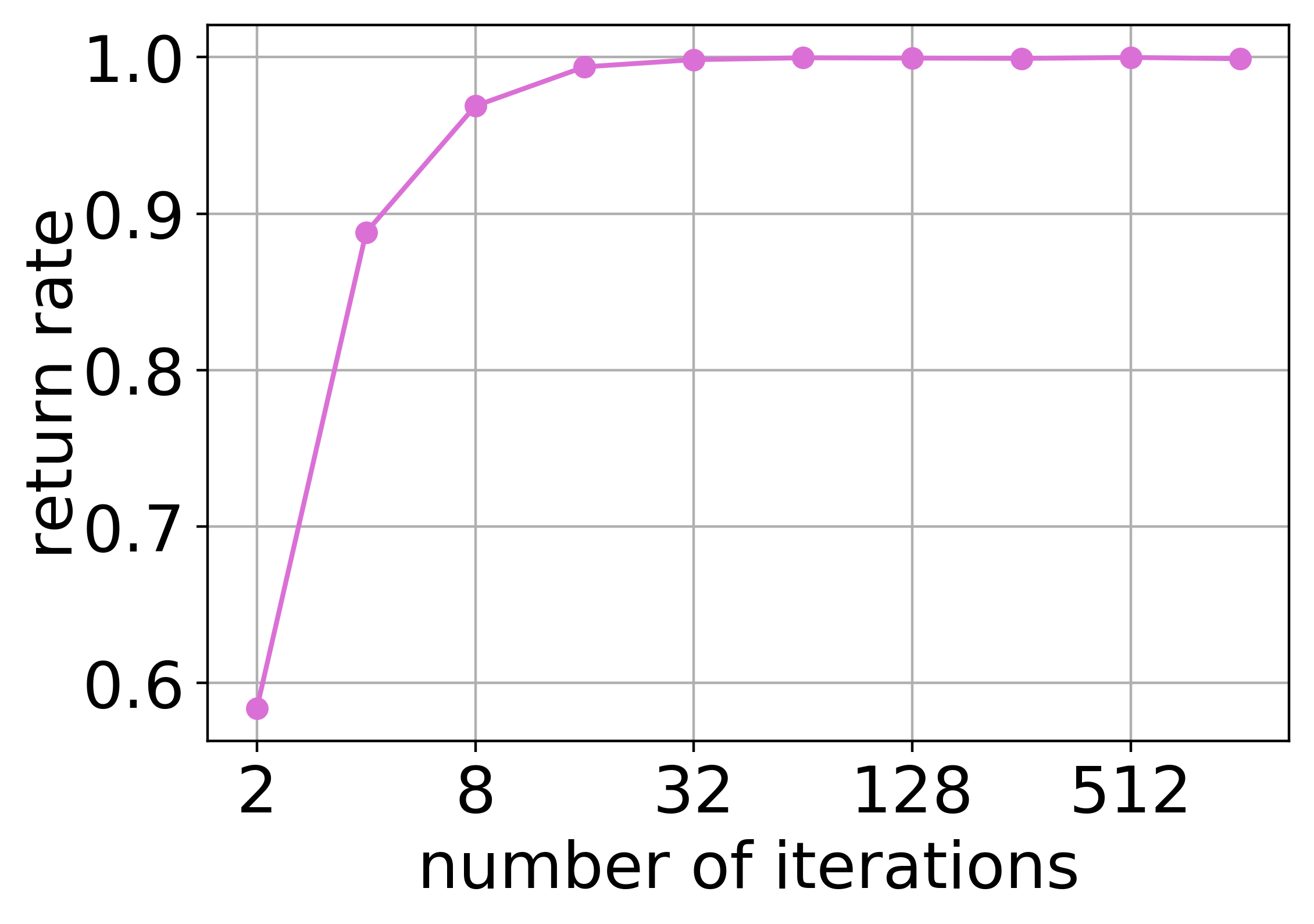}}
    \caption{(left): AUROC values as function of number of primary attack iterations for classifying between $D^{(0,j)}$ and $\bar{D}^{(k,j)}$. (right): Return rates vs.\ number of iterations for CIFAR10. }
    \label{fig:cifar_auroc}
\end{figure}

%
%
%
%

%
\begin{table}[t]
\centering
\scalebox{0.75}{
\begin{tabular}{||c||c||c||c||}
\cline{1-4}
\multicolumn{4}{||c||}{Numbers of samples in $\mX$ and $\bar{\mX}$} \\
\cline{1-4}
& two moons (2D) & CIFAR10 & ImageNet2012 \\
\cline{1-4}
$\ell_0$      & --   & 5000   &     500 \\
$\ell_2$      & 150  & 5000   &     500 \\
$\ell_\infty$ & --   & 1000   &     500 \\
\cline{1-4}
\end{tabular} }
\caption{Numbers of images / data points in $\mX$ and $\bar{\mX}$, respectively.}
\label{tab:Dsize}
\end{table}

In order to discuss the distance measures $D^{(0,j)}$ and $\bar{D}^{(k,j)}$ for different $\ell_p$ norms, we now alter the notation. 
Let $\gamma_p: I^n \to I^n$ 
be the function that maps $x \in \mX$ to its $\ell_p$ attacked counterpart for $p=0,2,\infty$. In order to demonstrate the separability of $\mX$ and $\gamma_p(\mXb)$ under a second $\ell_q$ attack, we compute the two scalar sets
\begin{equation} \label{eq:violins}
\begin{aligned}
D_q            = & \{ \mathit{dist}_q(x,\gamma_q(x)) \, : \,  x \in \mX \} \quad \text{and} \\
\bar{D}_{p,q}  = & \{ \mathit{dist}_q(\gamma_p(\bar{x}),\gamma_q(\gamma_p(\bar{x}))) \, : \,  \bar{x} \in \mXb \}
\end{aligned}
\end{equation}
for $q=0,2,\infty$. Besides CIFAR10 we consider ImageNet2012\footnote{\url{http://image-net.org}} with high resolution rgb images. Since CleverHans only provides $\ell_2$ attacks, we altered the code to the one provided with \cite{DBLP:journals/corr/CarliniW16a} and used the included pre-trained models (for further details, see \cref{sec:impl}). From now on we used default parameters and performed $1,\!000$ iterations for both attacks, respectively. For the number of images used, see \cref{tab:Dsize}.
Since the $\ell_p$ norms used in our tests are equivalent except for $\ell_0$, we expect that cross attacks, i.e., the case $p \neq q$ is supposed to yield a good separation of $D_{q}$ and $\bar{D}_{p,q}$ (cf.~\cref{eq:violins}). 
\Cref{tab:crossedAttacks} shows results for cross attacks with AUROC values up to $99.73\%$. Each column shows the detection performance of a norm $\ell_q$.
\begin{table}[t]
    \centering
    \scalebox{0.66}{
    \begin{tabular}{||c||c|c|c||c|c|c||}
        \cline{1-7}
        Attack & \multicolumn{3}{c||}{CIFAR10} & \multicolumn{3}{c||}{ImageNet2012}  \\
        \cline{1-7}
        2nd$\downarrow$ $\backslash{}$ 1st$\rightarrow$ & $p=0$ & $p=2$ & $p=\infty$ & $p=0$ & $p=2$ & $p=\infty$ \\
        \cline{1-7}
        $q=0$ & 91.86\% & 99.62\% & 99.64\% & 98.73\% & 99.30\% & 99.73\% \\
        $q=2$ & 91.25\% & 99.13\% & 99.17\% & 95.82\% & 96.04\% & 96.75\% \\
        $q=\infty$  & 89.23\% & 96.81\% & 96.29\% & 96.31\% & 94.55\% & 95.61\% \\
        \cline{1-7}
    \end{tabular}}
    \caption{AUROC values for the different cross attacks performed  to separate $D_{q}$ and $\bar{D}_{p,q}$, the $\ell_q$ norm is the one used for detection.}
    \label{tab:crossedAttacks}
\end{table}
\begin{table}[t]
\centering
\scalebox{0.75}{
\begin{tabular}{||c||c||c||c||c||}
\cline{1-5}
                      & dataset & $\ell_0$ & $\ell_2$  & $\ell_\infty$\\ 
\cline{1-5}
Returnees & CIFAR10 & $4994$ & $4997$ & $995$ \\ 
Percentage & CIFAR10 & $99.88\%$ & $99.94\%$ & $99.5\%$ \\ 
\cline{1-5}
Returnees & ImageNet2012 & $500$ & $486$ & $488$ \\ 
Percentage & ImageNet2012 & $100\%$ & $97.20\%$ & $97.6\%$ \\
\cline{1-5}
\end{tabular} }
\caption{Number of returnees as well as return rates. All primary attacks were successful.}
\label{tab:classFlips}
\end{table}
In both cases, for CIFAR10 and ImageNet2012, when comparing the different columns of both plots we observe a superiority of the $\ell_2$ norm. In our tests we observe that the $\ell_2$ norm also requires lowest computational effort, thus the $\ell_2$ norm might be favorable from both perspectives.
Noteworthily, there is also only a minor performance degradation when going from CIFAR10 to ImageNet2012 although the perturbations introduced by the $\ell_p$ attacks, in particular for $p=2$, are almost imperceptible, cf.\ also~\cref{fig:attcifar}. The inferiority of the $\ell_0$ attack might be due to the less granular iteration steps, changing pixel values maximally.
In \cref{tab:classFlips} we report the number of returned examples and return rates. 
For both datasets we observe strong return rates although there is no theoretical guarantee in a multi class setting. 
For a more detailed discussion of experiments with images including a comparison with the state-of-the-art, we refer to \cref{sec:further_exp}. Comparisons with other methods are difficult due to the variety in datasets and evaluation metrics considered. However, counter attacks with the CW attack seem to reach current state-of-the-art detection rates.



\section{Conclusion \& Outlook}

We presented a mathematical proof for asymptotically optimal detection of CW attacks via CW counter attacks. In numerical experiments we confirmed that the number of iterations indeed increases the separability of attacked and non-attacked images. Even when further relaxing the theoretical assumptions and considering default parameters, we still obtain detection rates of up to $99.73\%$ AUROC that are in the upper part of the spectrum of detection rates reported in the literature. 
For future work we plan to study the effect of introducing a confidence parameter in the primary CW attack. We expect that the treacherous efficiency of the CW attack can still be noticed in that case. Furthermore we plan to investigate under which conditions stationary points of the CW attack are local minimizers and to consider theoretically the cases $p \neq 2$. Our CleverHans branch is publicly available under \url{https://github.com/kmaag/cleverhans}.





\appendix

\section{Additional Theoretical Considerations}

\subsection{Isolated Stationary Points} \label{sec:isolated}

We now present the proof of the following theorem:

\begin{theorem} \label{theo:isolated}
Let $ x^* \in \cap_{j=1}^s Q_j,\; x \in (0,1)^n $ be a stationary point. If the corresponding gradients $ \{g_j\}_{j=1}^s $ are linearly independent then $ x^* $ is isolated for $ a $  large enough. 
\end{theorem}

\begin{proof}
Denote $ C = dist_p(x_0,x^*). $
A point $ x^* $ is an isolated stationary point if there exists $ \varepsilon > 0 $ such that $ x^* $ is the only stationary point in $ B(x^*,\varepsilon).$ Let us first assume that $ x^* \in \partial {\cal F}.$ Then $ \nabla_C f(x^*) = \{a g_1,\ldots,a g_s,0\}$ and there exist $ \lambda_1, \ldots, \lambda_s \geq 0, \; \sum_{j=1}^s \lambda_j \leq 1 $ such that 
\begin{equation} \label{stac1}
a \lambda_1 g_1 + \ldots a \lambda_s g_s = 2(x_0 - x^*).
\end{equation}
Assume now that $ x^* $ is not isolated. Let $ \tilde{x} \in \partial {\cal F} $ be another stationary point such that $ \dist(\tilde{x},x^*) = \varepsilon$ for arbitrary small $ \varepsilon >0,$ and $ x^*, \tilde{x}$ belong to the intersection of some set of polytopes, say $ x^*, \tilde{x} \in \cap_{i=1}^{\ell} Q_i, \; 1 \leq \ell \leq s. $ Then there exists a vector $ d $ such that $ \|d\|_2=1, d \subset \partial {\cal F}$ and $ \tilde{x} = x^* + \varepsilon d.$ Given that $ \tilde{x} $ is stationary, we have $ \nabla_C f(\tilde{x}) = \mathrm{conv}\{a g_1,a g_2, \ldots, ag_{\ell}, 0\}$ and there exists $ \beta_1, \ldots, \beta_{\ell}  \geq 0, \sum_{i=1}^{\ell} \beta_i \leq 1,  $ such that 
\begin{equation} \label{stac2}
a \beta_1 g_1 + \ldots + a \beta_{\ell} g_{\ell} = 2(x_0 - \tilde{x}) = 2 (x_0 - x^* - \varepsilon{}d). 
\end{equation}
Given that $ d \subset \partial {\cal F},$ we have $ d^T g_1 =0, \ldots, d^T g_{\ell} = 0. $
Let us now consider separately two cases, $ s = \ell $ and $ s > \ell. $
Assume first that $ \ell = s. $ Then, combining \cref{stac1} and \cref{stac2} we get
\begin{equation} \label{stac3}
    a (\lambda_1 - \beta_1) g_1 + a (\lambda_2 - \beta_2)  g_2 + \ldots +  a (\lambda_s - \beta_s) g_s = 2 \varepsilon d. 
\end{equation}
From \cref{stac3} we have 
\begin{align}
    & a (\lambda_1 - \beta) d^T g_1 + a (\lambda_2 - \beta_2) d^T g_2 + \ldots \nonumber  \\
    & \ldots + a (\lambda_s - \beta_s) d^T g_s
    = 2 \varepsilon \|d\|^2,
\end{align}
and the orthogonality of $ d $ implies 
\begin{equation}
    0 = 2 \varepsilon \|d\|^2 \, ,
\end{equation}
which is clearly not true. In the second case, $ \ell < s, $ we have $ d^T g_1 =0, \ldots, d^T g_{\ell} = 0 $ and from \cref{stac1} and \cref{stac2}, 
\begin{equation} \label{stac4}
a \lambda_{\ell+1} d^T g_{\ell+1} + \ldots + a \lambda_s d^T g_s = 2 \varepsilon \|d\|_2^2. 
\end{equation} 
On the other hand, \cref{stac1} yields $ a \lambda_{\ell+1} g_{\ell+1} + \ldots + a \lambda_s g_s = 2(x_0 - x^*) - a \sum_{i=1}^{\ell} \lambda_i g_i$, and therefore
\begin{eqnarray}
2 \varepsilon \|d\|_2^2 & = & a \sum_{i=\ell+1}^s \lambda_i d^T g_i \nonumber  \\
&= & 2 d^T (x_0 - x^* - a \sum_{j=1}^{\ell} \lambda_j g_j) \nonumber \\ 
&= & 2 d^T(x_0 - x^*) \, , 
\end{eqnarray}
which can not be true for arbitrary small $ \varepsilon $ as the right hand side expression does not depend on $ \varepsilon. $ 
 
Assume now that there exists $ \tilde{x} \in B(x^*,\varepsilon) \cap {\cal F}. $ As in that case we have $ f(\tilde{x}) = 0 $ and $ \nabla f(\tilde{x}) = 0$ we conclude that $ x_0 - \tilde{x} =0,$ which is clearly impossible. 
 
Finally, assume that there exists a stationary point $ \tilde{x} \in B(x^*,\varepsilon) \cap {\cal N}{\cal F}. $ We have two possibilities now. First, if $ \tilde{x} $ belongs to interior of one of the polytopes $ Q_1,\ldots,Q_s $ then $f $ is differentiable at $ \tilde{x} $ and $ F $ is locally a smooth strictly convex function. Thus $ \tilde{x} $ is an isolated stationary point, which is not possible as $ \tilde{x} \in B(x^*,\varepsilon). $ The remaining case to analyse is  $ \tilde{x} \in \cap_{j=1}^{s} Q_j.  $ Assume that $ a $ is large enough such that $ \sum_{j=1}^s \lambda_j = \delta \ll 1. $  Let $ \tilde{x} \in \cap_{j=1}^{\ell} Q_j, \; 1\leq \ell \leq s$ be another stationary point in the ball $B(x^*,\varepsilon).  $  In this case we have $ \nabla_C f(\tilde{x}) = \mathrm{conv}\{a g_1, \ldots, a g_{\ell}\}) $ and given that $ \tilde{x} \in S,$ there exist $ \beta_1,\ldots,\beta_{\ell} \geq 0, \sum_{j=1}^{s} \beta_j = 1, $ such that 
\begin{equation} \label{stac5}
a \beta_1 g_1 + \ldots a \beta_s g_{\ell} = 2(x_0 - \tilde{x}) \, . 
\end{equation}
Consider the following system of linear equations
\begin{equation} \label{stac6}
a \gamma_1 g_1 + \ldots \gamma_s g_s = 0 \, . 
\end{equation}
Given that $ g_1,\ldots,g_s$ are linearly independent and the system is overdetermined ({\bf $n \gg s$}), the only solution is $ \gamma_1=\ldots \gamma_n = 0.$ Using \cref{stac1} and \cref{stac5} we have
\begin{align}
    2(x^* - \tilde{x}) & =  2(x^* - x_0 + x_0 - \tilde{x}) \\
    & = a (\beta_1 - \lambda_1) g_1 + \ldots + a (\beta_s - \lambda_s) g_s \, , 
\end{align}
with $ \beta_i = 0 $ for $ \ell < i \leq s. $ Given that $ x^* - \tilde{x} $ is arbitrary small, $ \dist_2(x^*,\tilde{x}) = \varepsilon,$ the above system is arbitrary close to \cref{stac6} and hence as $ \varepsilon \to 0 $ we must have that its solution converges to the solution of \cref{stac6}, i.e.,
\begin{equation}
   \beta_1 - \lambda_1 \to 0, \ldots, \beta_s - \lambda_s \to 0 \, . 
\end{equation}
But $ \sum_{j=1}^s \lambda_j = \delta \ll 1 $ and $ \sum_{j=1}^s \beta_j = 1$ which is clearly impossible. 
So far, we have proved that any stationary point $ x^*  \in \partial {\cal F} $ is isolated. But the arguments used in the proof also cover the remaining cases, i.e., $ x^* \in {\cal F} $ and $ x^* \in {\cal N} {\cal F}. $ 
\end{proof}

\subsection{Information Theoretic Derivation of the AUROC} \label{sec:continuousAUROC}

Let $X$ and $Y$ be two real valued random variables with distribution $X_{*}P$ and $Y_{*}P$, respectively. We denote the cumulative density functions of $X$ and $Y$ by $F_{X}(\tau) = P(X \leq \tau)$ and $F_{Y}(\tau) = P(Y \leq \tau)$. A point on the receiver operator curve induced by a threshold $\tau$ measures, how well this threshold separates values of $Y$ from values of $X$. Suppose $Y$ is suspected to be statistically larger than $X$. Then the true positive rate for the recognition of $X$ with the simple threshold rule based on the threshold value $\tau$ is $P(X \leq \tau)$, whereas the false positive rate is $P(Y \leq \tau)$. The point of the ROC-curve is then given by $(F_{X}(\tau), F_{Y}(\tau))$ and variation of $\tau$ yields the entire curve. In order to determine the area under the receiver operator curve, we approximate the surface integral as indicated in \cref{fig:aurocderivation} and in the limit of the Riemann integral we obtain 
\begin{equation} \label{eq:informationtheoreticAUROC}
\begin{aligned}
    \text{AUROC}(Y,X) & = \int_{\mathbb{R}} F_{X}(\tau) \ d F_{Y}(\tau) \\
    & = \int_{\mathbb{R}} F_{X}(\tau) \ d Y_{*}P(\tau) \, ,
\end{aligned}
\end{equation}
where in the last step we replaced the Riemann integral with the more general Lebesque integral, which of course gives the same value. Note that any monotone function is of bounded variation, thus the Riemann integral representation of the AUROC always exists.
\begin{figure}
\centering
\scalebox{0.8}{
\begin{tikzpicture}
\draw[] (5,5) arc (-270:-180:5.0);
\draw [-latex] (0,0) -- (0,5.5);
\draw [-latex] (0,0) -- (5.5,0);
\draw [dotted,gray] (5,0) -- (5,5);
\draw [dotted,gray] (0,5) -- (5,5);
\draw [dotted,gray] (0,0) -- (5,5);
\draw [dotted,blue] (1,0) -- (1,3);
\draw [dotted,blue] (1.8,0) -- (1.8,3.8);
\draw [blue,fill=blue] (1.4,3.43) circle (0.05);
\node [blue] [rotate=45] at (1.4,3.8) {$F_{X}( \tau )$};
\node [blue] at (1.4,-0.3) {$ \Delta F(Y \leq \tau )$};
\node at (5,-0.5) {FPR = $F(Y \leq \tau )$};
\node [rotate=90] at (-0.5,5) {TPR = $F_{X}( \tau )$};
\end{tikzpicture}}
\caption{A sketch of the information theoretic definition of AUROC. $\Delta F(Y \leq \tau)$ denotes a segment on the horizontal axis. From the limit $\Delta F(Y \leq \tau) \to 0$ as in the Riemann integral, we obtain \cref{eq:informationtheoreticAUROC}.}
\label{fig:aurocderivation}
\end{figure}
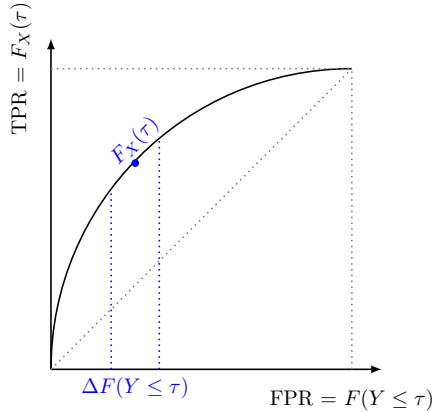

\section{Further Numerical Experiments and Details}
\label{sec:further_exp}

\begin{figure*}[t]
    \centering
    \includegraphics[trim=10 1 1 1,clip,width=0.74\textwidth]{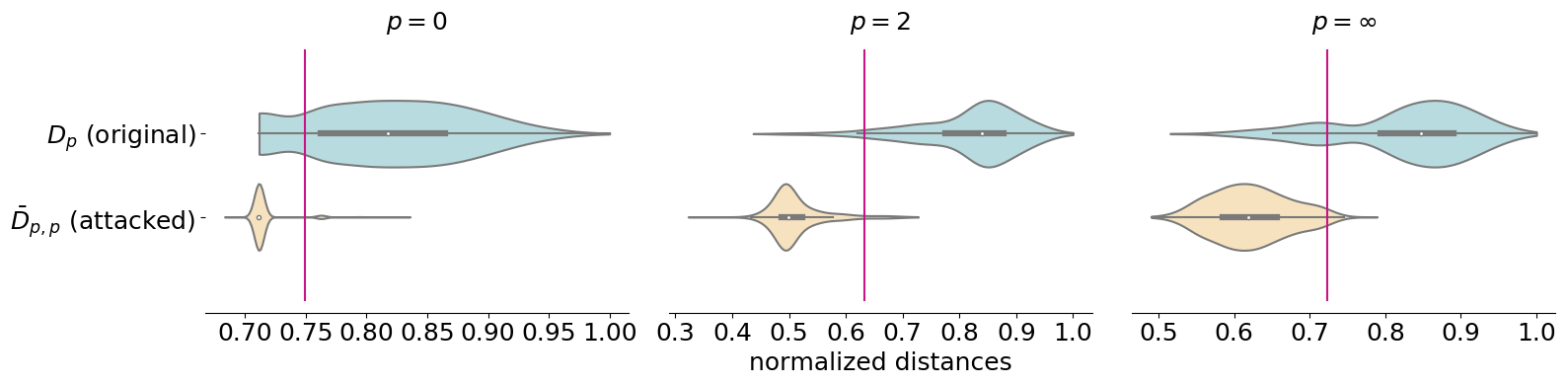}
    \includegraphics[width=0.25\linewidth]{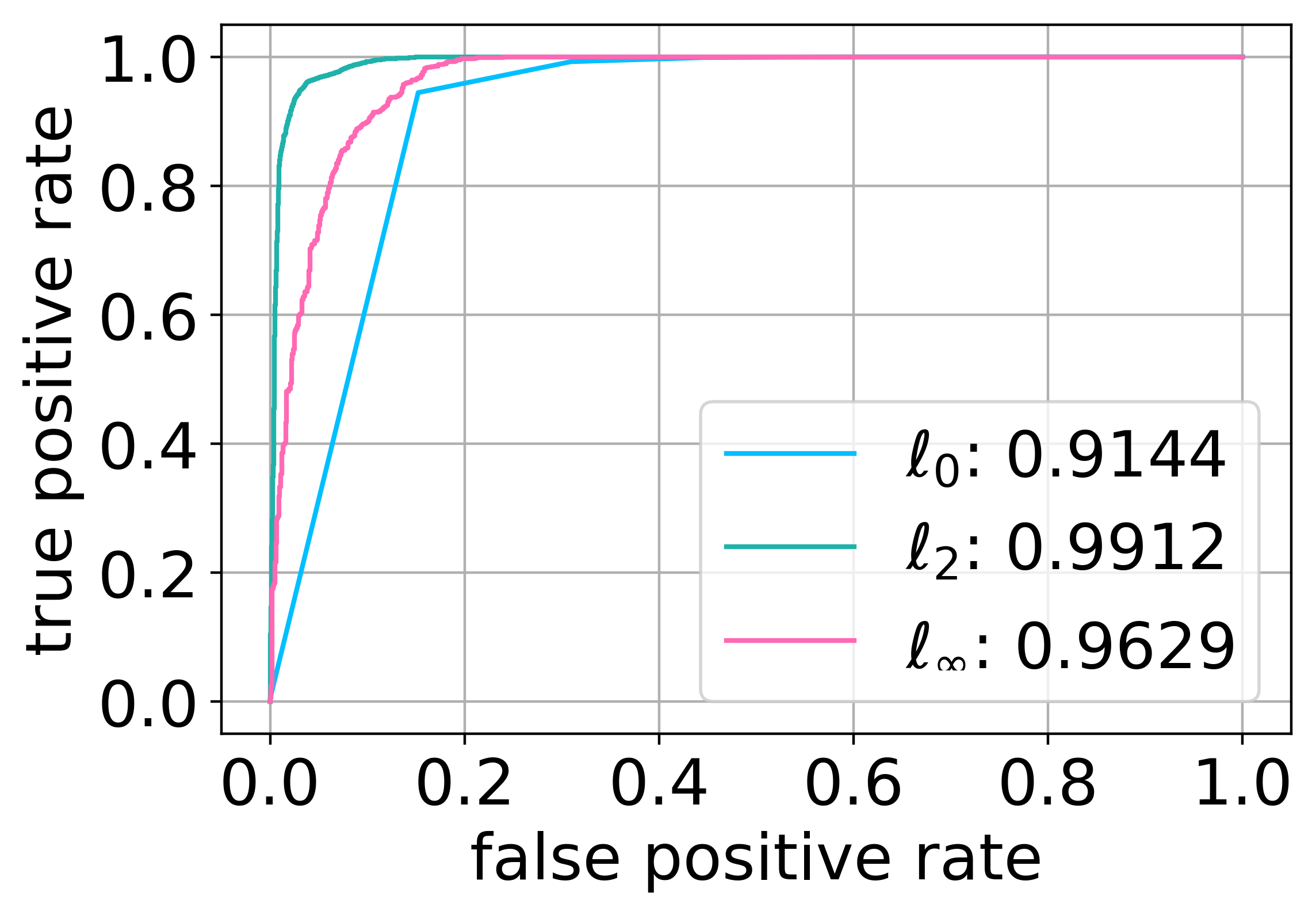}
    \caption{(left): Violin plots displaying the distributions of distances in $D_p$ (top) and $\bar{D}_{p,p}$ (bottom) from \cref{eq:violins} for the CIFAR10 dataset, $p=0,2,\infty$ in ascending order from left to right. ImageNet2012 violins look similar. (right): Corresponding ROC curves and AUROC values.}
    \label{fig:CIFAR10vio}
\end{figure*}

\subsection{Choice of Penalization Strength of the Primary Attack.} \label{sec:penalizationstrength}

In \cref{sec:theocons} we noted that the penalization strength $a$ needs to be large enough. For numerical experiments we need to derive a lower bound. A stationary point $x^*$ needs to fulfill 
\begin{equation}
2(x_0-x^*) = a \sum_i^s \lambda_i g_i(x)
\end{equation} 
for $\lambda_i\geq0$ and $\sum_{i=1}^t \lambda_i \leq 1 $ where $g_i$ is the gradient of the $i$th polytopes $Q_i$ and w.l.o.g.\ $x^* \in Q_i$, $i=1,\ldots,t$, $t\geq1$. Hence, a penalization strength $a$ is sufficient if it fulfills
\begin{equation}
2 \norm{x_0-x^*} \leq 2 \sqrt{d} \leq a \cdot c
\end{equation}
where $d$ is the input dimension and $c$ as in \cref{eq:defcC} the smallest gradient norm of all polytopes that intersect with the decision boundary. It follows that
\begin{equation}
a = \frac{2\sqrt{d}}{c}
\end{equation}
is sufficiently large.

\subsection{Technical Details}
\label{sec:twomoondim} \label{sec:impl}

In contrast to the 2D experiments in \cref{sec:numex} of the main text, we cannot compute the stationary points in higher dimensions anymore. In particular, the CleverHans framework does not allow to easily access the weights of the neural network. Therefore, in the 2D example we had to discretize the domain up to an order of $10^{-4}$ and compute decision and polytope boundaries via network inference. Hence, all presented numbers for the 2D case are only accurate up to the discretization error. In higher dimensions we cannot proceed similarly. Hence we adjust the following parameters: For the counter attack we set the initial learning rate $\alpha_{0} = 0.01$ (which is the same value as for the primary attack). For parameters $a$ and $b$ we run a binary search in range $0$ to $10^{10}$ (default in the CleverHans framework). 

For the CIFAR10 dataset \cite{CIFAR10} consisting of tiny $32\times32$ rgb images from 10 classes, containing 50k training and 10k test images, we used a network with 4 convolutional layers and two dense ones as included and used per default in the code. For the ImageNet2012\footnote{http://image-net.org} high resolution rgb images we used a pre-trained Inception network (\cite{inception}, trained on all $1,\!000$ classes) All networks were trained with default train/val/test splittings. For all tests we used default parameters.
For each of the two datasets, we randomly split the test set into two equally sized portions $\mX$ and $\bar{\mX}$. For one portion we compute the values $D^{(0,j)}$ and for the other one $D^{(0,k)}$ and $\bar{D}^{(k,j)}$, such that they refer to two distinct sets of original images. Since CW attacks can be computationally demanding, we chose the sample sizes for $\mX$ and $\bar{\mX}$ as stated in \cref{tab:Dsize} in the main text.

\subsection{CIFAR10 and ImageNet2012 $\ell_p$ Counter Attacks.}

We consider the case where the primary and counter attack share the same norm, i.e., $p=q$, and present a more detailed evaluation for CIFAR10 and ImageNet.
First we note that the success rates of $\gamma_p(\mXb)$, i.e., the percentages where $\kappa(\bar{x}) \neq \kappa(\gamma_p(\bar{x}))$ for $\bar{x} \in \mXb$, equal a $100\%$
for all attacks $\gamma_p$, $p=0,2,\infty$, see \cref{tab:successrates}. 

\begin{table}[t]
\centering
\scalebox{0.85}{
\begin{tabular}{||c||c||c||c||c||}
\cline{1-5}
\multicolumn{5}{||c||}{Success rates} \\
\cline{1-5}
& \multicolumn{2}{c||}{CIFAR10} & \multicolumn{2}{c||}{ImageNet2012} \\
\cline{1-5}
&                 1st attack & 2nd attack    & 1st attack  & 2nd attack \\ 
\cline{1-5}
$\ell_0$      & 100\% & 100\%   &     100\% & 100\% \\
$\ell_2$      & 100\% & 100\%   &     100\% & 100\% \\
$\ell_\infty$ & 100\% & 100\%   &     100\% & 100\% \\
\cline{1-5}
\end{tabular} }
\caption{Success rates of primary and counter attacks. An attack is considered successful if the predicted class of the attack's output is different from the predicted class of the attack's input.}
\label{tab:successrates}
\end{table}

%
%
The separability of original data $\mX$ and attacked data $\gamma_p(\mXb)$ is visualized by violin plots for $D_p$ and $\bar{D}_{p,p}$ in \cref{fig:CIFAR10vio} (left) for CIFAR10.
For ImageNet2012 we observe a similar behavior. For $p=0,2$, we observe well-separable distributions. For $p=\infty$ this separability seems rather moderate. When looking at absolute $\ell_\infty$ distances, we observe that the distribution peaks roughly at $0.1$. For $\ell_2$, the absolute distances are much higher. This is clear, since the $\ell_\infty$ norm is independent of the image dimensions whereas the other norms suffer from the curse of dimensionality. This can be omitted by re-normalizing the shrinkage coefficient for $\ell_p$ distance minimization. However it also underlines the punchline of our method, the more efficient the iterative minimization process, the better our detection method -- the efficiency becomes treacherous. Less efficient attacks are not suitable detectors, however such attacks can potentially be detected by other detection methods as they tend to be more perceptible. 

Next, we study the performance in terms of relative frequency of correctly classified attacked / original images for two different thresholds, as well as in terms of the area under receiver operator characteristic curve (AUROC, \cite{DBLP:conf/icml/DavisG06}) which is threshold independent.
In all further discussions, true positives are correctly classified attacked images and true negatives are correctly classified original (non-attacked) images. More precisely, we fix the following notations:
%
\begin{itemize}
    \item threshold: value $t$ used for discriminating whether $d \in D_{q}$ or $d \in \bar{D}_{p,q}$ \,,
    \item true  positive: $TP = | \{ d \in \bar{D}_{p,q} \, : \, d < t    \} |$ \,,
    \item false negative: $FN = | \{ d \in \bar{D}_{p,q} \, : \, d \geq t \} |$ \,,
    \item true  negative: $TN = | \{ d \in D_{q}    \, : \, d \geq t \} |$ \,,
    \item false positive: $FP = | \{ d \in D_{q}    \, : \, d < t    \} |$ \,,
    \item accuracy: $\frac{TP + TN}{TP + FN + TN + FP}$ \,,
    \item precision: $\frac{TP}{TP + FP}$ \,,
    \item recall: $\frac{TP}{TP + FN}$ \,.
\end{itemize}

\begin{table}[t]
\centering
\scalebox{0.63}{
\begin{tabular}{||c||c||c||c||c||c||c||}
\cline{1-7}
\multicolumn{7}{||c||}{Sensitivity \& Specificity values} \\
\cline{1-7}
& \multicolumn{3}{c||}{threshold choice no. 1} & \multicolumn{3}{c||}{threshold choice no. 2} \\
\cline{1-7}
                 & $\ell_0$ & $\ell_2$  & $\ell_\infty$ & $\ell_0$ & $\ell_2$  & $\ell_\infty$\\ 
\cline{1-7}
True Positive & $4,\!723$ & $4,\!817$ & $984$
& $4,\!723$ & $4,\!803$ & $982$ \\ 
False Negative & $277$ & $183$ & $16$
& $277$ & $197$ & $18$ \\ 
True Negative & $4,\!239$ & $4,\!792$ & $837$
& $4,\!239$ & $4,\!809$ & $840$ \\ 
False Positive & $761$ & $208$ & $163$
& $761$ & $191$ & $160$  \\
Accuracy & $\boldsymbol{0.8962}$ & $\boldsymbol{0.9609}$ & $\boldsymbol{0.9105}$
& $\boldsymbol{0.8962}$ & $\boldsymbol{0.9612}$ & $\boldsymbol{0.9110}$ \\
Recall & $\boldsymbol{0.9446}$ & $\boldsymbol{0.9634}$ & $\boldsymbol{0.9840}$
& $\boldsymbol{0.9446}$ & $\boldsymbol{0.9606}$ & $\boldsymbol{0.9820}$ \\
Precision & $\boldsymbol{0.8612}$ & $\boldsymbol{0.9586}$ & $\boldsymbol{0.8579}$
& $\boldsymbol{0.8612}$ & $\boldsymbol{0.9617}$ & $\boldsymbol{0.8599}$ \\
Threshold & $5.0126$ & $0.00891$ & $0.0324$
& $5.0126$ & $0.008225$ & $0.0321$ \\
\cline{1-7}
\end{tabular}}
\caption{Sensitivity \& Specificity values for the CIFAR10. Choice no.\ 1 of thresholds is obtained by maximizing $0.5 \times \textit{recall} + 0.25 \times \textit{accuracy} + 0.25 \times \textit{precision}$, choice no.\ 2 is obtained by maximizing accuracy. 
}
\label{tab:CIFAR10sensspecvalues}
\end{table}

%
In \cref{tab:CIFAR10sensspecvalues} we report those metrics for two choices of threshold computation for $p=q$. Choice no.\ 1 is obtained by $0.5 \times \textit{recall} + 0.25 \times \textit{accuracy} + 0.25 \times \textit{precision}$, choice no.\ 2 is obtained by maximizing accuracy. While choice no.\ 1 rather aims at detecting many adversarial attacks, therefore accepting to produce additional false positive detections, choice no.\ 2 attributes the same importance to the reduction of false positives and false negatives.
In accordance to our findings in \cref{fig:CIFAR10vio} (left) we observe high accuracy, recall and precision values. In particular the $\ell_2$ attack shows very strong detection results. The corresponding AUROC valus of up to $99.12\%$ are stated in \cref{fig:CIFAR10vio} (right).
Noteworthily, out of 10,000 images we only obtain 16 false positive detections, i.e., images that are falsely accused to be attacked, and 56 false negatives, i.e., attacked images that are not detected when using threshold choice no.\ 1.
For the $\ell_\infty$ attack we observe an over production of false positives. However, most of the 1,000 attacked images were still detected.

\subsection{CIFAR10 Targeted Attacks}
So far, all results presented have been computed for untargetted attacks. In principle a targeted attack $\gamma_p$ only increases the distances $\mathit{dist}_q ( \gamma_p(\bar{x}) , \bar{x})$ while the distance measures $\mathit{dist}_q ( \gamma_q(\gamma_p(\bar{x})) , \gamma_q(\bar{x}))$ corresponding to another untargeted attack $\gamma_q$ 
are supposed to remain unaffected. Thus, targeted attacks should be even easier to detect as confirmed by \cref{fig:TargettedCIFAR10}. 
However, intuitively the counter attack might be less successful in returning to the original class than in the untargeted case. For the untargeted primary attack on CIFAR10 with $p=q=2$ we obtained a return rate of $99.94\%$ (cf.\ \cref{tab:classFlips} in the main text), this number is indeed reduced to $73.32\%$ when performing a targeted primary attack.

\begin{figure}[tb]
    \centering
    \includegraphics[width=0.28\textwidth]{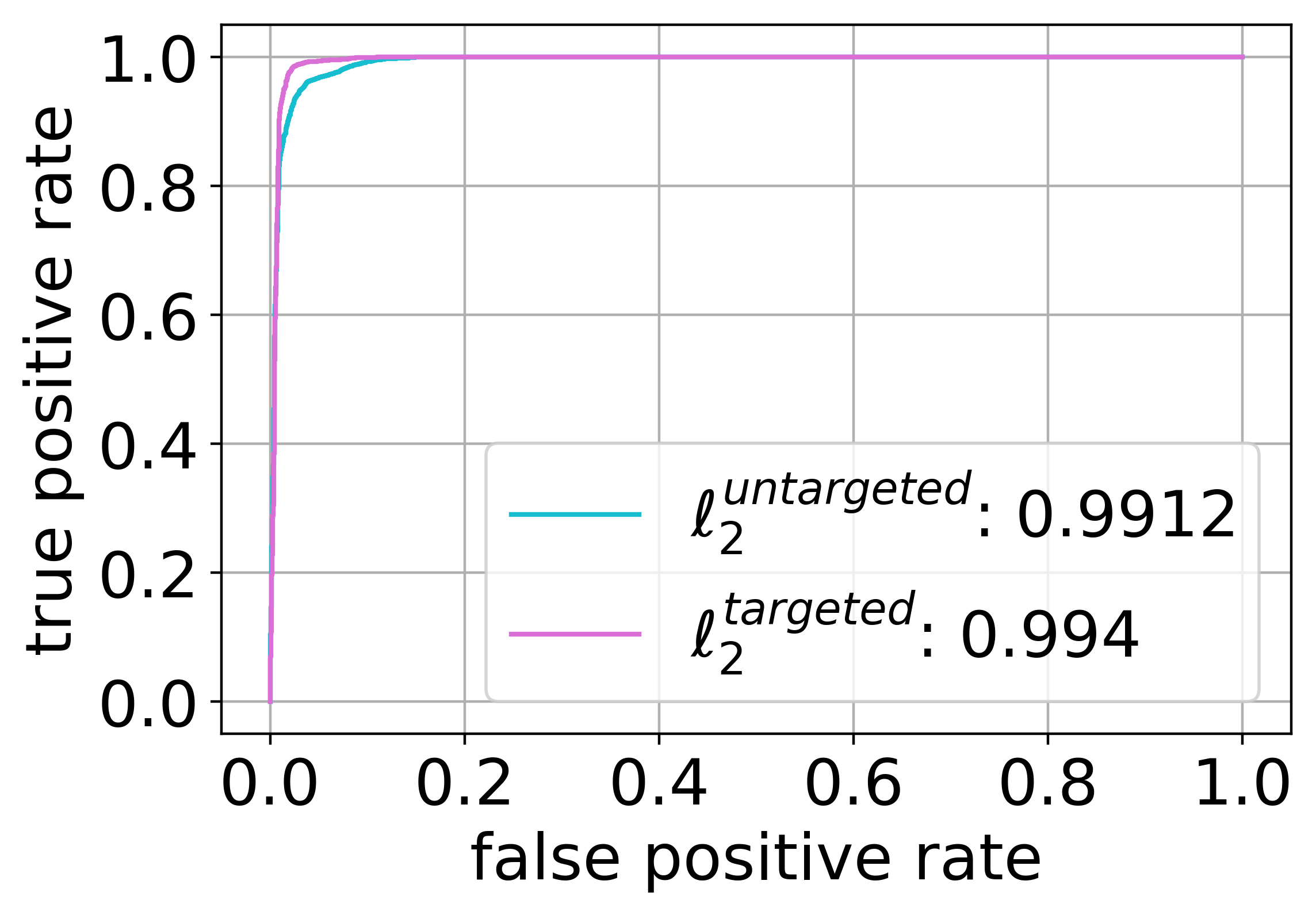}
    \caption{ROC curves and AUROC values for $\ell_{2}$ detections on CIFAR10 data where the first attack was once targeted and once untargeted.}
    \label{fig:TargettedCIFAR10}
\end{figure}

\begin{table}[t]
\centering
\scalebox{0.75}{
\begin{tabular}{||c|c|c|c|c||}
\cline{1-5}
method & metric & $\%$ & dataset & attack \\
\cline{1-5}
DBA & accuracy & 97.9 & CIFAR10 & CW $\ell_{2}$ \\
\cite{undercover} & & & & \\
\cline{1-5}
Defense-GAN & pm${}^*$ & 85 & CelebA & CW $\ell_{2}$ \\
\cite{DBLP:journals/corr/abs-1805-06605} 
& pm${}^*$ & 96 & F-MNIST & CW $\ell_{2}$ \\
\cline{1-5}
\cite{DBLP:journals/corr/abs-1902-04818}
& accuracy & 91.6 & CIFAR10 & CW $\ell_{2}$ \\
& accuracy & 96.1 & CIFAR10 & PGD $\ell_{2}$ \\
\cline{1-5}
I-defender & AUROC & 79.3 & CIFAR10 & PGD $\ell_{2}$ \\
\cite{Zheng:2018:RDA:3327757.3327888} & & & & \\
\cline{1-5}
\cite{DBLP:journals/corr/abs-1905-11475}
& AUROC & 99.7 & CIFAR10 & PGD $\ell_{2}$ \\
\cline{1-5}
\textbf{ours}  & accuracy & 96.1 & CIFAR10  & CW $\ell_{2}$ \\
 & AUROC & 99.1 & CIFAR10 & CW $\ell_{2}$ \\
\cline{1-5}
\end{tabular} }
\caption{Comparison with state of the art detection methods. *: the shorthand pm stands for performance maintenance and denotes the network's accuracy under attack divided by the original accuracy. Note that $100\%$ pm does not require a 100\% detection accuracy.}
\label{tab:SOTA}
\end{table}

\subsection{Comparison with State-of-the-Art}

Recent works report their numbers in a high heterogeneity w.r.t.\ considered attacks, evaluation metrics and datasets. This makes a clear comparison difficult. \Cref{tab:SOTA} summarizes a comparison with state-of-the-art detection methods. Our approach is situated in the upper part of the spectrum. PDG $\ell_2$ denotes the projected gradient descent method where gradients are normalized in an $\ell_2$ sense, see  \cite{Madry2017TowardsDL}. This attack is often considered in tests, however for most detection methods PDG $\ell_2$ might be easier to detect than the CW attack as PDG $\ell_2$ does not aim at minimizing the perturbation strength. For the detection by attack (DBA) approach \cite{undercover} which is similar to ours, it was observed empirically that eager attacks like CW are easier to detect than others. In light of our theoretical statements this observation is to be expected. The results presented by \cite{Worzyk1,Worzyk2} are limited to MNIST where the detection of attacks is simpler. Their finindgs are in line with those in \cite{undercover}, for the setting where both attacks are CW attacks, the authors of \cite{Worzyk1,Worzyk2} state a return rate of 99.6\% which is in line with what we observe.

\printbibliography

@article{DBLP:journals/corr/abs-1805-06605,
  author    = {Pouya Samangouei and
               Maya Kabkab and
               Rama Chellappa},
  title     = {Defense-GAN: Protecting Classifiers Against Adversarial Attacks Using
               Generative Models},
  journal   = {CoRR},
  volume    = {abs/1805.06605},
  year      = {2018},
  url       = {http://arxiv.org/abs/1805.06605},
  archivePrefix = {arXiv},
  eprint    = {1805.06605},
  bibsource = {dblp computer science bibliography, https://dblp.org}
}

@article{DBLP:journals/corr/abs-1904-00689,
  author    = {Olga Taran and
               Shideh Rezaeifar and
               Taras Holotyak and
               Slava Voloshynovskiy},
  title     = {Defending against adversarial attacks by randomized diversification},
  journal   = {CoRR},
  volume    = {abs/1904.00689},
  year      = {2019},
  url       = {http://arxiv.org/abs/1904.00689},
  archivePrefix = {arXiv},
  eprint    = {1904.00689},
  bibsource = {dblp computer science bibliography, https://dblp.org}
}

@article{DBLP:journals/corr/PapernotMWJS15,
  author    = {Nicolas Papernot and
               Patrick D. McDaniel and
               Xi Wu and
               Somesh Jha and
               Ananthram Swami},
  title     = {Distillation as a Defense to Adversarial Perturbations against Deep
               Neural Networks},
  journal   = {CoRR},
  volume    = {abs/1511.04508},
  year      = {2015},
  url       = {http://arxiv.org/abs/1511.04508},
  archivePrefix = {arXiv},
  eprint    = {1511.04508},
  bibsource = {dblp computer science bibliography, https://dblp.org}
}

@article{Goodfellow2014ExplainingAH,
  title={Explaining and Harnessing Adversarial Examples},
  author={Ian J. Goodfellow and Jonathon Shlens and Christian Szegedy},
  journal={CoRR},
  year={2014},
  volume={abs/1412.6572}
}

@article{Szegedy2013IntriguingPO,
  title={Intriguing properties of neural networks},
  author={Christian Szegedy and Wojciech Zaremba and Ilya Sutskever and Joan Bruna and Dumitru Erhan and Ian J. Goodfellow and Rob Fergus},
  journal={CoRR},
  year={2013},
  volume={abs/1312.6199}
}

@article{DBLP:journals/corr/KurakinGB16a,
  author    = {Alexey Kurakin and
               Ian J. Goodfellow and
               Samy Bengio},
  title     = {Adversarial Machine Learning at Scale},
  journal   = {CoRR},
  volume    = {abs/1611.01236},
  year      = {2016},
  url       = {http://arxiv.org/abs/1611.01236},
  archivePrefix = {arXiv},
  eprint    = {1611.01236},
  bibsource = {dblp computer science bibliography, https://dblp.org}
}

@article{Metzen2017OnDA,
  title={On Detecting Adversarial Perturbations},
  author={Jan Hendrik Metzen and Tim Genewein and Volker Fischer and Bastian Bischoff},
  journal={ArXiv},
  year={2017},
  volume={abs/1702.04267}
}

@inproceedings{Lee2018DefensiveDM,
  title={Defensive denoising methods against adversarial attack},
  author={Sungyoon Lee and Saerom Park and Jaewook Lee},
  year={2018},
  booktitle={24th ACM SIGKDD Conference on Knowledge Discovery and Data Mining}
}

@article{Madry2017TowardsDL,
  title={Towards Deep Learning Models Resistant to Adversarial Attacks},
  author={Aleksander Madry and Aleksandar Makelov and Ludwig Schmidt and Dimitris Tsipras and Adrian Vladu},
  journal={ArXiv},
  year={2017},
  volume={abs/1706.06083}
}

@article{DBLP:journals/corr/CarliniW16a,
  author    = {Nicholas Carlini and
               David A. Wagner},
  title     = {Towards Evaluating the Robustness of Neural Networks},
  journal   = {CoRR},
  volume    = {abs/1608.04644},
  year      = {2016},
  url       = {http://arxiv.org/abs/1608.04644},
  archivePrefix = {arXiv},
  eprint    = {1608.04644},
  bibsource = {dblp computer science bibliography, https://dblp.org}
}

@article{Chen2017EADEA,
  title={EAD: Elastic-Net Attacks to Deep Neural Networks via Adversarial Examples},
  author={Ping Chen and Yash Sharma and Huan Zhang and Jinfeng Yi and Cho-Jui Hsieh},
  journal={ArXiv},
  year={2017},
  volume={abs/1709.04114}
}

@article{DBLP:journals/corr/Moosavi-Dezfooli15,
  author    = {Seyed{-}Mohsen Moosavi{-}Dezfooli and
               Alhussein Fawzi and
               Pascal Frossard},
  title     = {DeepFool: a simple and accurate method to fool deep neural networks},
  journal   = {CoRR},
  volume    = {abs/1511.04599},
  year      = {2015},
  url       = {http://arxiv.org/abs/1511.04599},
  archivePrefix = {arXiv},
  eprint    = {1511.04599},
  bibsource = {dblp computer science bibliography, https://dblp.org}
}

@article{DBLP:journals/corr/abs-1803-00940,
  author    = {Aaditya Prakash and
               Nick Moran and
               Solomon Garber and
               Antonella DiLillo and
               James A. Storer},
  title     = {Protecting {JPEG} Images Against Adversarial Attacks},
  journal   = {CoRR},
  volume    = {abs/1803.00940},
  year      = {2018},
  url       = {http://arxiv.org/abs/1803.00940},
  archivePrefix = {arXiv},
  eprint    = {1803.00940},
  bibsource = {dblp computer science bibliography, https://dblp.org}
}

@article{DBLP:journals/corr/abs-1803-05787,
  author    = {Zihao Liu and
               Qi Liu and
               Tao Liu and
               Yanzhi Wang and
               Wujie Wen},
  title     = {Feature Distillation: DNN-Oriented {JPEG} Compression Against Adversarial
               Examples},
  journal   = {CoRR},
  volume    = {abs/1803.05787},
  year      = {2018},
  url       = {http://arxiv.org/abs/1803.05787},
  archivePrefix = {arXiv},
  eprint    = {1803.05787},
  bibsource = {dblp computer science bibliography, https://dblp.org}
}

@article{Simonyan14c,
    author       = "Simonyan, K. and Zisserman, A.",
    title        = "Very Deep Convolutional Networks for Large-Scale Image Recognition",
    journal      = "CoRR",
    volume       = "abs/1409.1556",
    year         = "2014"
}

@incollection{NIPS2012_4824,
  author = {Krizhevsky, Alex and Sutskever, Ilya and Hinton, Geoffrey E.},
  booktitle = {Advances in Neural Information Processing Systems 25},
  editor = {Pereira, F. and Burges, C. J. C. and Bottou, L. and Weinberger, K. Q.},
  pages = {1097--1105},
  publisher = {Curran Associates, Inc.},
  title = {ImageNet Classification with Deep Convolutional Neural Networks},
  year = {2012}
}

@article{DBLP:journals/corr/HeZRS15,
  author    = {Kaiming He and
               Xiangyu Zhang and
               Shaoqing Ren and
               Jian Sun},
  title     = {Deep Residual Learning for Image Recognition},
  journal   = {CoRR},
  volume    = {abs/1512.03385},
  year      = {2015},
  url       = {http://arxiv.org/abs/1512.03385},
  archivePrefix = {arXiv},
  eprint    = {1512.03385},
  bibsource = {dblp computer science bibliography, https://dblp.org}
}

@article{DBLP:journals/corr/abs-1902-04818,
  author    = {Kevin Roth and
               Yannic Kilcher and
               Thomas Hofmann},
  title     = {The Odds are Odd: {A} Statistical Test for Detecting Adversarial Examples},
  journal   = {CoRR},
  volume    = {abs/1902.04818},
  year      = {2019},
  url       = {http://arxiv.org/abs/1902.04818},
  archivePrefix = {arXiv},
  eprint    = {1902.04818},
  bibsource = {dblp computer science bibliography, https://dblp.org}
}

@article{DBLP:journals/corr/AthalyeS17,
  author    = {Anish Athalye and
               Logan Engstrom and
               Andrew Ilyas and
               Kevin Kwok},
  title     = {Synthesizing Robust Adversarial Examples},
  journal   = {CoRR},
  volume    = {abs/1707.07397},
  year      = {2017},
  url       = {http://arxiv.org/abs/1707.07397},
  archivePrefix = {arXiv},
  eprint    = {1707.07397},
  bibsource = {dblp computer science bibliography, https://dblp.org}
}

@article{DBLP:journals/corr/abs-1711-01991,
  author    = {Cihang Xie and
               Jianyu Wang and
               Zhishuai Zhang and
               Zhou Ren and
               Alan L. Yuille},
  title     = {Mitigating adversarial effects through randomization},
  journal   = {CoRR},
  volume    = {abs/1711.01991},
  year      = {2017},
  url       = {http://arxiv.org/abs/1711.01991},
  archivePrefix = {arXiv},
  eprint    = {1711.01991},
  bibsource = {dblp computer science bibliography, https://dblp.org}
}

@InProceedings{pmlr-v97-wang19i,
  title = 	 {On the Convergence and Robustness of Adversarial Training},
  author = 	 {Wang, Yisen and Ma, Xingjun and Bailey, James and Yi, Jinfeng and Zhou, Bowen and Gu, Quanquan},
  booktitle = 	 {Proceedings of the 36th International Conference on Machine Learning},
  pages = 	 {6586--6595},
  year = 	 {2019},
  editor = 	 {Chaudhuri, Kamalika and Salakhutdinov, Ruslan},
  volume = 	 {97},
  series = 	 {Proceedings of Machine Learning Research},
  address = 	 {Long Beach, California, USA},
  month = 	 {09--15 Jun},
  publisher = 	 {PMLR},
  url = 	 {http://proceedings.mlr.press/v97/wang19i.html}
}

@inproceedings{Zheng:2018:RDA:3327757.3327888,
 author = {Zheng, Zhihao and Hong, Pengyu},
 title = {Robust Detection of Adversarial Attacks by Modeling the Intrinsic Properties of Deep Neural Networks},
 booktitle = {Proceedings of the 32Nd International Conference on Neural Information Processing Systems},
 series = {NIPS'18},
 year = {2018},
 location = {Montr\&\#233;al, Canada},
 pages = {7924--7933},
 numpages = {10},
 url = {http://dl.acm.org/citation.cfm?id=3327757.3327888},
 acmid = {3327888},
 publisher = {Curran Associates Inc.},
 address = {USA},
}

@article{DBLP:journals/corr/abs-1711-00117,
  author    = {Chuan Guo and
               Mayank Rana and
               Moustapha Ciss{\'{e}} and
               Laurens van der Maaten},
  title     = {Countering Adversarial Images using Input Transformations},
  journal   = {CoRR},
  volume    = {abs/1711.00117},
  year      = {2017},
  url       = {http://arxiv.org/abs/1711.00117},
  archivePrefix = {arXiv},
  eprint    = {1711.00117},
  bibsource = {dblp computer science bibliography, https://dblp.org}
}

@article{DBLP:journals/corr/abs-1907-05587,
  author    = {Steven Chen and
               Nicholas Carlini and
               David A. Wagner},
  title     = {Stateful Detection of Black-Box Adversarial Attacks},
  journal   = {CoRR},
  volume    = {abs/1907.05587},
  year      = {2019},
  url       = {http://arxiv.org/abs/1907.05587},
  archivePrefix = {arXiv},
  eprint    = {1907.05587},
  bibsource = {dblp computer science bibliography, https://dblp.org}
}

@article{DBLP:journals/corr/LiangLSLSW17,
  author    = {Bin Liang and
               Hongcheng Li and
               Miaoqiang Su and
               Xirong Li and
               Wenchang Shi and
               Xiaofeng Wang},
  title     = {Detecting Adversarial Examples in Deep Networks with Adaptive Noise
               Reduction},
  journal   = {CoRR},
  volume    = {abs/1705.08378},
  year      = {2017},
  url       = {http://arxiv.org/abs/1705.08378},
  archivePrefix = {arXiv},
  eprint    = {1705.08378},
  bibsource = {dblp computer science bibliography, https://dblp.org}
}

@inproceedings{8599691,
author={S. {Jha} and U. {Jang} and S. {Jha} and B. {Jalaian}},
booktitle={MILCOM 2018 - 2018 IEEE Military Communications Conference (MILCOM)},
title={Detecting Adversarial Examples Using Data Manifolds},
year={2018},
volume={},
number={},
pages={547-552},
doi={10.1109/MILCOM.2018.8599691},
ISSN={},
month={Oct},}

@InProceedings{Carrara_2018_ECCV_Workshops,
author = {Carrara, Fabio and Becarelli, Rudy and Caldelli, Roberto and Falchi, Fabrizio and Amato, Giuseppe},
title = {Adversarial examples detection in features distance spaces},
booktitle = {The European Conference on Computer Vision (ECCV) Workshops},
month = {September},
year = {2018}
}

@article{DBLP:journals/corr/GrosseMP0M17,
  author    = {Kathrin Grosse and
               Praveen Manoharan and
               Nicolas Papernot and
               Michael Backes and
               Patrick D. McDaniel},
  title     = {On the (Statistical) Detection of Adversarial Examples},
  journal   = {CoRR},
  volume    = {abs/1702.06280},
  year      = {2017},
  url       = {http://arxiv.org/abs/1702.06280},
  archivePrefix = {arXiv},
  eprint    = {1702.06280},
  bibsource = {dblp computer science bibliography, https://dblp.org}
}

@article{DBLP:journals/corr/abs-1905-11475,
  author    = {Xuwang Yin and
               Soheil Kolouri and
               Gustavo K. Rohde},
  title     = {Divide-and-Conquer Adversarial Detection},
  journal   = {CoRR},
  volume    = {abs/1905.11475},
  year      = {2019},
  url       = {http://arxiv.org/abs/1905.11475},
  archivePrefix = {arXiv},
  eprint    = {1905.11475},
  bibsource = {dblp computer science bibliography, https://dblp.org}
}

@inproceedings{DBLP:conf/icml/DavisG06,
  author    = {Jesse Davis and
               Mark Goadrich},
  title     = {The relationship between Precision-Recall and {ROC} curves},
  booktitle = {Machine Learning, Proceedings of the Twenty-Third International Conference
               {(ICML} 2006), Pittsburgh, Pennsylvania, USA, June 25-29, 2006},
  pages     = {233--240},
  year      = {2006},
  url       = {http://doi.acm.org/10.1145/1143844.1143874},
  doi       = {10.1145/1143844.1143874},
  bibsource = {dblp computer science bibliography, https://dblp.org}
}

@misc{inception,
  author = {Google},
  title = {Tensorflow Inception Network},
  howpublished = {\url{http://download.tensorflow.org/models/image/imagenet/inception-2015-12-05.tgz}},
  year={2015}
}

@book{UnderstandingML,
  author = {Shalev-Shwartz, Shai and Ben-David, Shai},
  ee = {http://www.cambridge.org/de/academic/subjects/computer-science/pattern-recognition-and-machine-learning/understanding-machine-learning-theory-algorithms},
  isbn = {978-1-10-705713-5},
  pages = {I-XVI, 1-397},
  publisher = {Cambridge University Press},
  title = {Understanding Machine Learning - From Theory to Algorithms.},
  year = 2014
}

@article{HeinPolytopes,
  author    = {Francesco Croce and
               Maksym Andriushchenko and
               Matthias Hein},
  title     = {Provable Robustness of ReLU networks via Maximization of Linear Regions},
  journal   = {CoRR},
  volume    = {abs/1810.07481},
  year      = {2018},
  url       = {http://arxiv.org/abs/1810.07481},
  archivePrefix = {arXiv},
  eprint    = {1810.07481},
  bibsource = {dblp computer science bibliography, https://dblp.org}
}

@book{Sch12,
  title={Introduction to Piecewise Differentiable Equations},
  author={Scholtes, S.},
  isbn={9781461443407},
  lccn={2012942479},
  series={SpringerBriefs in Optimization},
  url={https://books.google.de/books?id=EpAfMFZ3vtUC},
  year={2012},
  publisher={Springer New York}
}

@article{Solodov,
  title={Error stability properties of Generalized Gradient-Type Algorithms},
  author={Sodolov, M.~V. and Zavriev, S.~K.},
  journal={JOTA},
  year={1998},
  volume={98},
  pages={663-680},
  publisher={Plenum Publishing Corporation}
}

@book{Clarke,
title={Optimization and Nonsmooth Analysis},
author={Clarke, F.C.},
publisher={SIAM},
URL = {https://epubs.siam.org/doi/abs/10.1137/1.9781611971309.ch2},
eprint = {https://epubs.siam.org/doi/pdf/10.1137/1.9781611971309.ch2},
year={1990}
}

@article{REN2020346,
title = "Adversarial Attacks and Defenses in Deep Learning",
journal = "Engineering",
volume = "6",
number = "3",
pages = "346 - 360",
year = "2020",
issn = "2095-8099",
doi = "https://doi.org/10.1016/j.eng.2019.12.012",
url = "http://www.sciencedirect.com/science/article/pii/S209580991930503X",
author = "Kui Ren and Tianhang Zheng and Zhan Qin and Xue Liu"
}

@article{DBLP:journals/ijautcomp/XuMLDLTJ20,
  author    = {Han Xu and
               Yao Ma and
               Haochen Liu and
               Debayan Deb and
               Hui Liu and
               Jiliang Tang and
               Anil K. Jain},
  title     = {Adversarial Attacks and Defenses in Images, Graphs and Text: {A} Review},
  journal   = {Int. J. Autom. Comput.},
  volume    = {17},
  number    = {2},
  pages     = {151--178},
  year      = {2020},
  url       = {https://doi.org/10.1007/s11633-019-1211-x},
  doi       = {10.1007/s11633-019-1211-x},
  bibsource = {dblp computer science bibliography, https://dblp.org}
}

@InProceedings{undercover,
author="Zhou, Qifei
and Zhang, Rong
and Wu, Bo
and Li, Weiping
and Mo, Tong",
editor="Chen, Liqun
and Li, Ninghui
and Liang, Kaitai
and Schneider, Steve",
title="Detection by Attack: Detecting Adversarial Samples by Undercover Attack",
booktitle="Computer Security -- ESORICS 2020",
year="2020",
publisher="Springer International Publishing",
address="Cham",
pages="146--164",
isbn="978-3-030-59013-0"
}

@article{DBLP:journals/corr/LuIF17,
  author    = {Jiajun Lu and
               Theerasit Issaranon and
               David A. Forsyth},
  title     = {SafetyNet: Detecting and Rejecting Adversarial Examples Robustly},
  journal   = {CoRR},
  volume    = {abs/1704.00103},
  year      = {2017},
  url       = {http://arxiv.org/abs/1704.00103},
  archivePrefix = {arXiv},
  eprint    = {1704.00103},
  bibsource = {dblp computer science bibliography, https://dblp.org}
}

@book{Ferguson,
title={A Course in Large Sample Theory},
author={Ferguson, T.S.},
publisher={Springer US},
year={1996}
}

@inproceedings{Worzyk1,
  author    = {Nils Worzyk and
               Oliver Kramer},
  title     = {Properties of adv-1 - Adversarials of Adversarials},
  booktitle = {26th European Symposium on Artificial Neural Networks, {ESANN} 2018,
               Bruges, Belgium, April 25-27, 2018},
  year      = {2018},
  url       = {http://www.elen.ucl.ac.be/Proceedings/esann/esannpdf/es2018-164.pdf},
  bibsource = {dblp computer science bibliography, https://dblp.org}
}

@inproceedings{Worzyk2,
  author    = {Nils Worzyk and
               Oliver Kramer},
  title     = {Adversarials \({}^{\mbox{-1}}\): Defending by Attacking},
  booktitle = {2018 International Joint Conference on Neural Networks, {IJCNN} 2018,
               Rio de Janeiro, Brazil, July 8-13, 2018},
  pages     = {1--8},
  publisher = {{IEEE}},
  year      = {2018},
  url       = {https://doi.org/10.1109/IJCNN.2018.8489630},
  doi       = {10.1109/IJCNN.2018.8489630},
  bibsource = {dblp computer science bibliography, https://dblp.org}
}

@article{CIFAR10,
  title={Learning Multiple Layers of Features from Tiny Images},
  author={Krizhevsky, Alex},
  year={2009}
}

@article{cleverhans2018,
  title={Technical Report on the CleverHans v2.1.0 Adversarial Examples Library},
  author={Nicolas Papernot and Fartash Faghri and Nicholas Carlini and
  Ian Goodfellow and Reuben Feinman and Alexey Kurakin and Cihang Xie and
  Yash Sharma and Tom Brown and Aurko Roy and Alexander Matyasko and
  Vahid Behzadan and Karen Hambardzumyan and Zhishuai Zhang and
  Yi-Lin Juang and Zhi Li and Ryan Sheatsley and Abhibhav Garg and
  Jonathan Uesato and Willi Gierke and Yinpeng Dong and David Berthelot and
  Paul Hendricks and Jonas Rauber and Rujun Long},
  journal={arXiv preprint arXiv:1610.00768},
  year={2018}
}

\end{document}